\documentclass[11pt]{article}
\usepackage[letterpaper,margin=1in]{geometry}
\usepackage[T1]{fontenc}
\usepackage{lmodern}
\usepackage{microtype}
\usepackage{graphicx}
\usepackage{booktabs}
\usepackage{amsmath,amssymb,amsfonts,amsthm,bm}
\usepackage[round,authoryear]{natbib}
\usepackage{caption}
\usepackage{url}
\usepackage[hidelinks]{hyperref}
\newtheorem{theorem}{Theorem}
\newtheorem{proposition}{Proposition}
\newtheorem{lemma}{Lemma}

\title{CADS: Measuring the Incremental Contribution of\\LLM Proposers in Sequential Biological Screens}
\author{
  Guanxing Chen$^{1}$, Nihui Shao$^{1}$, Haohuai He$^{2}$\\
  Zhenchao Tang$^{3}$, Pengwei Hu$^{4}$, Zhi-An Huang$^{1,*}$\\[0.65em]
  \small $^{1}$City University of Hong Kong (Dongguan)\\
  $^{2}$The Hong Kong Polytechnic University\\
  $^{3}$Tencent AI Lab\\
  $^{4}$Xinjiang Technical Institute of Physics and Chemistry, Chinese Academy of Sciences\\[0.5em]
  $^{*}$Corresponding author.
}
\date{}
\makeatletter
\renewcommand{\maketitle}{%
  \begin{center}
    \rule{\textwidth}{1.2pt}\par
    \vspace{0.8em}
    {\bfseries\fontsize{18pt}{22pt}\selectfont \@title\par}
    \vspace{0.7em}
    \rule{\textwidth}{0.5pt}\par
    \vspace{0.9em}
    {\@author\par}
  \end{center}
  \vspace{-0.4em}
}
\makeatother

\hypersetup{
  pdftitle={CADS: Measuring the Incremental Contribution of LLM Proposers in Sequential Biological Screens},
  pdfauthor={Guanxing Chen, Nihui Shao, Haohuai He, Zhenchao Tang, Pengwei Hu, Zhi-An Huang},
  pdfsubject={Certified measurement of LLM proposer contributions in sequential biological screens},
  pdfkeywords={LLM agents, biological screens, false discovery rate, e-values, sequential testing}
}

\begin{document}
\maketitle

\begin{abstract}
Large language models are increasingly used to guide sequential biological screens, yet standard evaluations rely on raw hit counts that conflate a proposer's value with the discovery capacity of the underlying algorithmic harness. Adaptive sampling and corrupted experimental readouts also leave standard feedback comparisons without a false-discovery guarantee. We present Certified Agentic Discovery Screens (CADS), an error-controlled measurement framework that quantifies a proposer's incremental value against an identical LLM-free loop. By integrating capped likelihood-ratio bounds, e-processes, and complete-family e-BH, CADS controls the terminal false discovery rate under the stated null and corruption conditions despite arbitrary dependence, adaptive sampling, and optional stopping. In DepMap-backed simulations, CADS exposes two evaluation artifacts: the LLM-free harness alone clears the feedback gate, and raw origin counts overstate LLM contributions because LLM picks displace fallback discoveries. On the initial 20 seeds, supplying tested-candidate history to Qwen2.5-7B shifts its net contribution from $+0.35$, with a $95\%$ interval covering zero, to $+2.90$, with the interval above zero. This initial net effect did not persist in a 50-seed matched-quota replication. These results support net gain over an LLM-free loop, measured between separately certified sets, as the appropriate measure of proposer contribution.
\end{abstract}

\begin{figure}[t]
\centering
\includegraphics[width=\columnwidth]{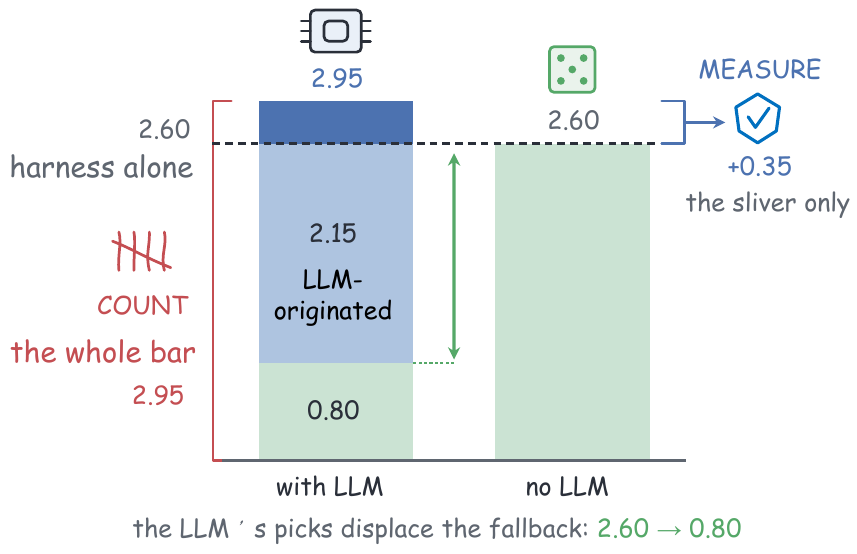}
\caption{Counting is not measuring. The same loop (an LLM proposer with a re-selection policy
and a random fallback) reads two ways. \emph{Counting} its hits credits every one to the proposer and
carries no error guarantee. \emph{Measuring} reads it through an FDR certificate, reporting the
proposer's net over an identical no-LLM loop, each loop's own declarations being FDR-controlled. They
disagree sharply: the no-LLM loop alone scores $2.60$, so the net is only $+0.35$ under a thin prompt.}
\label{fig:motivation}
\end{figure}

\section{Introduction}
\label{sec:intro}

LLM agents can propose genes or compounds, receive assay summaries, and redirect later rounds of a
biological screen \citep{roohani2025biodiscoveryagent,wainrib2026labfeedback}. These systems are often
judged by how many selected candidates become hits. A discovery loop, however, also contains a
re-selection policy and a fallback that fills slots the LLM leaves empty. Raw hit counts can therefore
credit the proposer with discoveries the rest of the loop would have made on its own
(Fig.~\ref{fig:motivation}), and they provide no control of false discoveries.

The missing quantity is the proposer's added contribution relative to the same loop without the LLM.
Origin attribution is not enough: an LLM-selected hit may displace a hit that re-selection or random
fallback would otherwise find. The relevant counterfactual is therefore the same discovery loop with
the proposer removed, not a count of hits attached to the proposer inside the full loop.

Reliable measurement must also separate proposal quality from readout contamination: poor candidate
choices reduce discovery power, whereas corrupted numerical readouts can invalidate the test. Existing
bandit multiple-testing work separates adaptive exploration from e-process evidence and extends it to
corrupted settings \citep{xu2021bandit,zhou2026robust}, and concurrent work finds non-LLM baselines
matching LLM feedback gains under cumulative hit-rate evaluation \citep{gupta2025bayesopt}. What
remains is a paired, error-controlled screen that measures incremental proposer contribution while
preserving the full registered family at the terminal decision.

To address this measurement gap, we introduce Certified Agentic Discovery Screens (CADS), an
FDR-controlled measuring instrument. CADS runs the full discovery loop beside an otherwise identical
proposer-removed loop and reports their paired difference in certified discoveries. A capped
likelihood-ratio increment limits the influence of each readout, e-processes accumulate evidence across
rounds, and complete-family e-BH makes the terminal decision over every registered candidate
(Fig.~\ref{fig:architecture}). Under the stated clean-null and conditional-contamination assumptions,
this controls FDR despite adaptive selection, optional stopping, and dependence. These statistical
components come from established bandit-multiple-testing and robust-e-value methods
\citep{xu2021bandit,zhou2026robust,saha2024huberrobust}; the novelty is the screen contract and its use
for paired measurement, not new FDR-control theory.

DepMap-backed simulations first test the instrument: across the contamination grid and eight reference
settings, CADS stays below the target FDR while naive BH and uncapped LR e-BH can exceed it. Evaluation
of three open-weight models then reveals what raw counts hide. The no-LLM re-selection loop clears the
standard feedback gate on its own, and origin-attributed discoveries overstate the proposer's net
because some displace hits the fallback would otherwise find. Prompt state is also measurable:
listing already-tested candidates changes how many valid candidates every model delivers, although a
$50$-seed replication does not reproduce the initial effect size on the net.

Our contributions are:
\begin{itemize}
\item \textbf{A certificate used for paired measurement.} Under the stated clean-null and conditional
contamination assumptions, CADS controls FDR under adaptive allocation, stopping, and dependence, then
measures proposer contribution as the difference between separately certified discovery sets.
\item \textbf{A proposer-removed counterfactual.} Removing only the proposer isolates discoveries
made by re-selection and fallback and accounts for valid fallback discoveries displaced by LLM picks.
\item \textbf{Prompt information as a measured factor.} Holding the loop fixed while exposing
already-tested candidates identifies a model-dependent change in valid proposals, although the
$50$-seed replication does not reproduce the initial net effect size.
\end{itemize}

\begin{figure}[t]
\centering
\includegraphics[width=\textwidth]{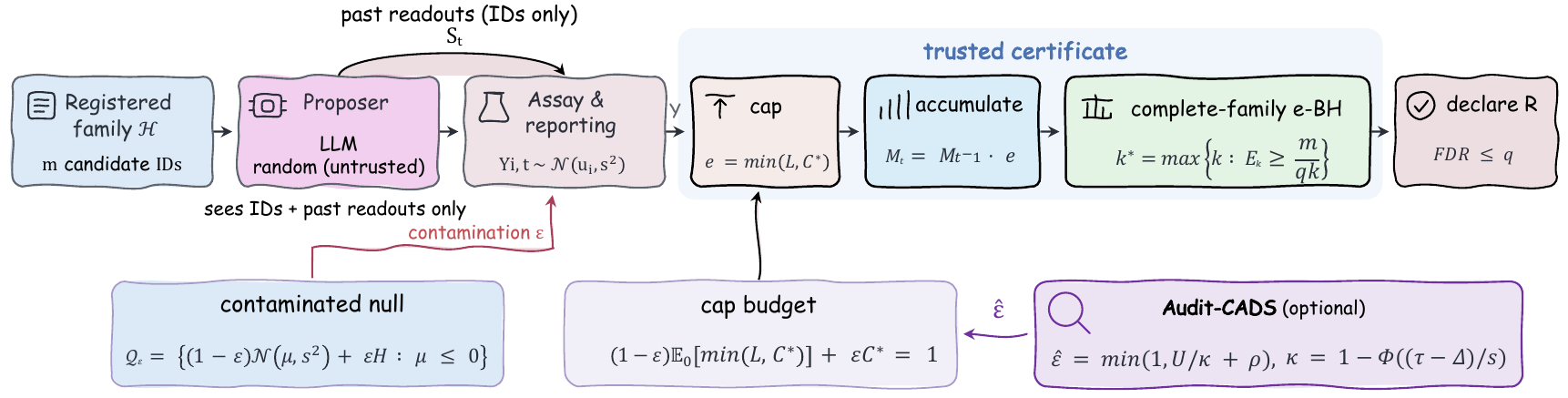}
\caption{CADS separates proposal, evidence, declaration, and evaluation. An arbitrary proposer
returns candidate identifiers from a view containing only identifiers and a protocol-released summary
of past readouts, which in our runs is high- and low-reading identifier lists.
The assay/reporting channel returns numerical evidence. A trusted certificate caps likelihood ratios,
accumulates e-processes, fills unobserved registered candidates with neutral e-value one, and applies
e-BH to the complete family. Reference effects and labels sit below the truth-isolation barrier and are
evaluation-only. Audit-CADS is a separate, conditional path requiring a stationary shared channel, audit
sensitivity at least $\kappa$, and test--audit mismatch at most $\rho$.}
\label{fig:architecture}
\end{figure}

\section{Related Work}
\label{sec:related}

\paragraph{Adaptive multiple testing.}
The closest statistical foundation is the modular bandit-testing framework of
\citet{xu2021bandit}, which separates arbitrary exploration from e-process
evidence and e-BH decisions under optional stopping, cross-arm dependence,
simultaneous queries, and multiple agents. \citet{zhou2026robust} extends this
setting to Huber and adversarial corruption through bounded-influence robust
e-processes. CADS adopts these principles. Its contribution is an operational
screening contract: a fixed registry, a proposer interface, truth isolation,
complete-family terminal checks, and a paired proposer-removed loop. The capped
simple likelihood ratio and conditional randomized-control extension implement
that contract; they are not a new general bandit-FDR theory.

\paragraph{e-values under contamination.}
e-BH controls FDR from valid e-values under arbitrary dependence
\citep{wang2022fdr}, and e-processes support optional continuation through
nonnegative supermartingales
\citep{ramdas2023savi,grunwald2024safe,ville1939etude}. Huber's censored
likelihood-ratio test supplies the robust primitive
\citep{huber1964robust,huber1965robust}; later work develops anytime robust
likelihood ratios and confidence sequences
\citep{saha2024huberrobust,wang2023huber}. Our proof specializes this machinery
to a one-sided Gaussian family. Online FDR procedures
\citep{foster2008alphainvesting,javanmard2018online,ramdas2018saffron,xu2023onlineevalues}
still require valid input evidence, so changing the terminal decision rule does
not by itself repair contaminated readouts.

\paragraph{LLM-driven experimental feedback.}
BioDiscoveryAgent motivates LLM-guided perturbation design
\citep{roohani2025biodiscoveryagent}, and laboratory-feedback studies
\citep{wainrib2026labfeedback} together with GEARS, PerturbQA, and GenePT
\citep{roohani2024gears,wu2025perturbqa,chen2025genept} provide the adjacent
biological setting. The nearest empirical comparison is concurrent work by
\citet{gupta2025bayesopt}: under cumulative hit-rate evaluation, non-LLM
baselines match or beat LLM feedback, and part of the hybrid gain comes from
showing the agent which candidates were already tested. CADS does not claim
these qualitative observations as new. It turns the no-LLM loop into a paired,
error-controlled measurement, accounts for displaced fallback discoveries,
and measures prompt information on certified outcomes. Conformal selection
\citep{jin2023selection,qin2025revamping,asiaee2026partial} controls
shortlisting errors under exchangeability, but does not replace sequential
evidence when candidates can be observed adaptively.

\section{Registered Screen and Certificate}
\label{sec:method}

\paragraph{Method overview.}
CADS treats proposal, evidence accumulation, and declaration as separate
interfaces. It first registers the candidate family and the information
available to any proposer, then converts each numerical readout into a bounded
e-value under a one-sided contaminated-null model. These increments form
anytime-valid e-processes, and complete-family e-BH produces one terminal
declared set. Conditional Audit-CADS optionally replaces a fixed contamination
cap with a predictable cap estimated from blinded controls under additional
assumptions. The same certificate is then run independently in paired loops to
measure what an LLM proposer adds.

\paragraph{Interface and target.}
CADS separates proposal, numerical evidence, and terminal declaration under one registered
screening contract. Before any readout, register a family $\mathcal H=\{H_i:i\in[m]\}$. At round $t$, a proposer observes
a view $V_t$ containing candidate identifiers and structured past readouts, then returns a batch
$S_t\subseteq[m]$. The choice is predictable with respect to past information and may be randomized,
heuristic, or LLM-driven. Assaying $i\in S_t$ returns a numerical $Y_{i,t}$. Benchmark effects and
reference labels remain in a private evaluation oracle and never enter $V_t$, the acquisition score,
or the certificate. At a terminal, possibly data-dependent stopping time $\tau$, the certificate
returns $R\subseteq[m]$ and targets
\begin{equation}
\mathrm{FDR}(R)=\mathbb E\!\left[\frac{|R\cap\mathcal N|}{|R|\vee1}\right]\le q,
\label{eq:fdr}
\end{equation}
where $\mathcal N$ is the unknown set of true nulls. The guarantee is for this terminal registered
family, not for hypotheses created after registration or for an uncorrected union of declarations
made at multiple times.

\paragraph{One-sided contaminated null.}
Reference labels are threshold-derived, so the clean null is composite. Conditional on the past, a
null readout has a law in
\begin{equation}
\mathcal Q_{\varepsilon}
=\left\{(1-\varepsilon)\mathcal N(\mu,s^2)+\varepsilon H:
\mu\le0,\ H\ \text{arbitrary}\right\}.
\label{eq:nullfamily}
\end{equation}
The bound is per read and may hold for a history-dependent choice of $\mu$ and $H$. It does not assert
that an LLM's proposal mistakes are draws from this channel. A channel that
keeps the same marginal contamination rate but violates the conditional bound
voids the guarantee. Likewise, misspecifying the working alternative costs
power, whereas misspecifying the clean null can cost validity. Because the cap
budget below is an equality, setting
$\bar\varepsilon=\varepsilon$ leaves no assumption margin. Over-stating
$\varepsilon$ is the frozen cap's only headroom, and it protects against some
location and scale errors rather than arbitrary shape errors. Appendices~E
and~F, and the assumption stress tests below, quantify these boundaries.

\paragraph{Deployment contract.}
The guarantee is conditional, and most of the conditions are the laboratory's rather than the
certificate's. CADS returns FDR $\le q$ over the registered family at a data-dependent stopping time,
for an arbitrary proposer, provided: (C1) the family is registered before the campaign and no
hypothesis is added during it; (C2) the proposer has no direct access to private oracle effects, labels,
or oracle-derived rankings, seeing only registered identifiers and protocol-released readouts; (C3) the certificate reads the numerical channel, not
the agent's account of it; (C4) the clean null lies in the declared family $\mathcal N(\mu,s^2)$,
$\mu\le0$, at the stated $s$, so shape is specified and not only location and scale; a wider clean-law
family may be declared instead if the cap is re-solved for its worst case; and (C5)
per-read contamination is at most $\bar\varepsilon$ conditional on the past. C1--C3 are enforced by
construction and are covered by the released contract tests; C4 and C5 are the lab's to justify.
Audit-CADS calibrates C5 from blinded controls under its own extra assumptions (\S\ref{sec:audit}); it
does not remove or verify it. What CADS does not return is biological validity, replication, or utility:
it bounds one error rate over one family, and a certified set is a shortlist that has been priced, not a
set of facts.

Let $f_0=\mathcal N(0,s^2)$ be the boundary null, $f_1=\mathcal N(\theta,s^2)$ for fixed $\theta>0$,
and $L=f_1/f_0$. For $0<\varepsilon<1$, define $C^\ast$ by
\begin{equation}
(1-\varepsilon)\mathbb E_{f_0}[\min\{L,C^\ast\}]+\varepsilon C^\ast=1,
\label{eq:capbudget}
\end{equation}
and set $e(y)=\min\{L(y),C^\ast\}$. For $\varepsilon=0$ we use the uncapped LR. In the Gaussian
location family, $L$ and $e$ are nondecreasing and the clean expectation has a closed form. If
$t(C)=s^2\log(C)/\theta+\theta/2$, then
\begin{equation}
\mathbb E_{f_0}[\min\{L,C\}]
=\Phi\!\left(\frac{t(C)-\theta}{s}\right)
+C\left\{1-\Phi\!\left(\frac{t(C)}{s}\right)\right\}.
\label{eq:closedform}
\end{equation}

\begin{theorem}[Composite-null robust increment]
\label{thm:increment}
The capped increment defined by \eqref{eq:capbudget} satisfies
$\sup_{Q\in\mathcal Q_\varepsilon}\mathbb E_Q[e(Y)]\le1$.
\end{theorem}
\emph{Proof sketch.} For the clean component, $e$ is nondecreasing, so
$\mathcal N(\mu,s^2)\preceq_{\mathrm{st}}\mathcal N(0,s^2)$ for $\mu\le0$ implies
$\mathbb E_\mu[e]\le\mathbb E_{f_0}[e]$. For the contamination component, $e\le C^\ast$ pointwise.
Equation~\eqref{eq:capbudget} then gives the result. Appendix~A gives the full argument and the
closed-form calculation.

\paragraph{Anytime accumulation and complete-family e-BH.}
When candidate $i$ is sampled, multiply its current evidence by $e(Y_{i,t})$; when it is not sampled,
multiply by one. Under the conditional version of \eqref{eq:nullfamily}, the resulting $M_{i,t}$ is a
nonnegative supermartingale. Thus $E_i=M_{i,\tau}$ is an e-value at the adaptive stopping time. For an
unobserved candidate, $E_i=1$. Sort all $m$ values $E_{(1)}\ge\cdots\ge E_{(m)}$, set
\begin{equation}
k^\ast=\max\{k:E_{(k)}\ge m/(qk)\},
\label{eq:ebh}
\end{equation}
and return the top $k^\ast$.

\begin{theorem}[Registered-family CADS]
\label{thm:cads}
Under the conditional null family \eqref{eq:nullfamily}, CADS controls the terminal FDR in
\eqref{eq:fdr} for every predictable proposer and stopping time, without assumptions on dependence
among candidates.
\end{theorem}
This is the established e-process/e-BH bandit argument \citep{xu2021bandit,wang2022fdr}: the adaptive
allocation affects which multipliers are revealed but not their conditional null expectation, and
e-BH satisfies the pointwise false-discovery-proportion (FDP) bound
$\mathrm{FDP}(R)\le(q/m)\sum_{i\in\mathcal N}E_i$. The full family size $m$ is essential. Applying
e-BH only to the adaptively observed subset changes the threshold and is not the procedure covered by
Theorem~\ref{thm:cads}. The complete proof is in Appendix~A.

\subsection{Conditional Audit-CADS}
\label{sec:audit}

Fixed-$\varepsilon$ CADS needs a conservative upper bound. Audit-CADS replaces this input only under
a stronger identifiable model, and the assumptions have to be stated at the strength the proof uses.
Blinded randomized controls traverse the same stationary reporting channel as test candidates. Given
the channel parameters, their flags are i.i.d.\ Bernoulli with mean $p$, and the whole flag sequence
is independent of the test potential outcomes and of any test-side randomization; separate streams are
not enough to assume, so we assume the independence itself. Each harmful audit corruption is flagged
with probability at least $\kappa>0$, which \S\ref{sec:kappa} derives from a declared corruption-support class rather than assumes (clean false
positives only raise $p$, which only makes the bound more conservative). The mismatch condition is
\emph{conditional}, not a marginal rate: almost surely, for every null $i$ and round $t$,
\begin{equation}
\begin{aligned}
\Pr\!\left(\text{read }(i,t)\text{ corrupted}\mid\mathcal F_{t-1},\mathcal A\right)
&\le\varepsilon_{\mathrm{test}}\\
&\le\varepsilon_{\mathrm{audit}}+\rho
\end{aligned}
\label{eq:mismatch}
\end{equation}
with $\mathcal F_{t-1}$ the test history and $\mathcal A$ the whole audit sequence,
for pre-specified $\rho\ge0$. A scalar marginal rate would not do: the all-or-nothing channel above
has the same marginal $0.2$, yet on its corrupted branch three cap-saturating reads give
$C^*(0.2)^3=11.81>1/(q-\delta)=11.11$ at $q=0.10,\delta=0.01$, so the null is declared. Condition
\eqref{eq:mismatch} is exactly what rules that out, and it is condition C5
of the deployment contract rather than anything the audit can verify.
Let $U_t$ be a time-uniform $(1-\delta)$ upper confidence sequence for $p$ from a beta-binomial
mixture martingale. Before test round $t$, use
\begin{equation}
\widehat\varepsilon_t
=\min\{1,U_{t-1}/\kappa+\rho\}.
\label{eq:auditcap}
\end{equation}

\begin{theorem}[Conditional Audit-CADS]
\label{thm:audit}
Let $0<\delta<q<1$. Assume C1--C4 and the hypotheses of Theorem~\ref{thm:cads}, and in addition the
shared-channel, flag-sensitivity, conditional-mismatch \eqref{eq:mismatch}, and audit/test independence
assumptions above. Then the predictable cap in \eqref{eq:auditcap}, followed by e-BH at level
$q-\delta$ on the complete registered family, controls FDR at $q$.
\end{theorem}
On the simultaneous coverage event, $p\le U_t$, while
$p\ge\kappa\varepsilon_{\mathrm{audit}}$; hence every cap upper-bounds test contamination before it is
used. Theorem~\ref{thm:cads} then gives conditional FDR at most $q-\delta$; the complement contributes
at most $\delta$, and the full proof with its audit confidence sequence is in Appendix~A. This theorem makes no claim for undetectable corruptions without a justified
$\kappa$, nonstationary target-specific channels beyond $\rho$, or controls selected after seeing
their outcomes.

\subsection{Where $\kappa$ Comes From}
\label{sec:kappa}

A $\kappa$ that is merely assumed would assume away the problem, so we do not assume one. Nor do we
claim to recover $\varepsilon$ in general: the contamination fraction of an unrestricted Huber
neighbourhood is not identifiable when the contaminating law is unspecified
\citep{chen2016general}, and adapting to an unknown fraction carries a proven
price \citep{luo2024adaptive}. We therefore declare the contamination class in advance and compute
$\kappa$ from it. Let the channel flag a read iff $y>\tau$, and declare
$\mathcal H_\Delta=\{h:\operatorname{supp}(h)\subseteq[\Delta,\infty)\}$, i.e.\ corruption displaces a
read upward by at least $\Delta$.

\begin{proposition}[Derived flag sensitivity]
\label{prop:kappa}
Let independent $Z\sim N(0,s^2)$ and $b\sim h$ define a corrupted read $y=Z+b$, and write
$\kappa_h=\Pr(y>\tau\mid\text{corrupted})=\mathbb E_b[g(b)]$ with $g(b)=1-\Phi((\tau-b)/s)$. Since
$g'(b)=\phi((\tau-b)/s)/s>0$, the integrand is increasing, so $\kappa_h\ge g(\Delta)$ for every
$h\in\mathcal H_\Delta$ and
\[
\inf_{h\in\mathcal H_\Delta}\kappa_h=\kappa(\Delta,\tau,s)=1-\Phi\!\left(\frac{\tau-\Delta}{s}\right),
\]
attained, not merely approached, at the degenerate $h=\delta_\Delta$, which the class contains.
\end{proposition}

$\kappa$ is thus a property of the declared channel, computable before any data arrive, and the
declaration is falsifiable. What makes it a declaration rather than a formality is that $\Delta$ must
be finite: every finite $b$ is flagged with positive probability $g(b)>0$, so any class bounded below
admits a positive $\kappa$, and sensitivity collapses to zero only for a class that permits corruption
arbitrarily far below $\tau$. Such a class buys nothing ($\kappa\to0$ sends
$\widehat\varepsilon_t\to1$ and the cap to one, preserving validity and destroying power), which is
the honest failure mode.

What CADS removes is the free sensitivity constant, not the declared
corruption-support assumption behind it: $\kappa$ is derived conditional on
$\mathcal H_\Delta$, and neither $\Delta$ nor $\rho$ is estimated from data
here.

\paragraph{Paired measurement targets.}
\label{sec:measurement-targets}
To measure proposer contribution, the full loop and an otherwise identical
proposer-removed loop each run CADS independently. Each declared set is
therefore separately FDR-controlled. The proposer's \emph{net contribution} is
the paired, post-hoc difference in oracle-scored true declarations between
those loops; it is an estimand, not itself an FDR-controlled set. We distinguish
this net from the number of declarations first originated by the proposer,
because proposer choices can displace fallback discoveries. Prompt effects are
separate paired contrasts between fixed prompt conditions.

\section{Reference-Backed Evaluation}
\label{sec:experiments}

\subsection{Experimental Protocol}

\paragraph{Frozen protocol.}
We evaluate CADS in sequential simulations backed by public biological reference effects; they are
not prospective wet-lab campaigns. The primary registry contains $6000$ genes from DepMap 26Q1
ACH-000446, with $u_i=-\mathrm{effect}_i-0.5$ and reference hit $u_i>0$
\citep{tsherniak2017depmap,dempster2021chronos}. A label-free random priority introduces 500 fresh
candidates, and each of 15 continuation rounds re-observes the top 70\% of observed candidates before
adding another 500. A clean read is $Y\sim\mathcal N(u_i,0.5^2)$; independently with probability
$\varepsilon$, an upward value drawn uniformly from $[2,4]$ is added. We use working alternative
$\theta=0.7$. Terminal decisions use the complete registry. Primary grids use 1000 paired seeds and
additional settings 200, each mean carrying an MC-SE and a 95\% Monte Carlo interval in the released
seed-level artifacts.

The cross-setting evaluation uses the same code and parameters on one T-cell Perturb-seq
DOP1B/Rest reference \citep{zhu2025tcellperturb}, one PRISM ACH-000827 compound-sensitivity reference
\citep{corsello2020prism}, and six DepMap 24Q4 cell-line references with dependency-probability labels.
Effects and labels define the signal backbone and evaluation oracle only; repeated reads, adaptive
allocation, and corruption are simulated.

\paragraph{Comparators.}
Naive BH uses terminal Gaussian $p$-values from the adaptively collected reads; it is a failure
diagnostic, not a valid adaptive baseline. Uncapped LR e-BH uses a valid e-process only at
$\varepsilon=0$. The close robust comparator follows Zhou--Liu Eq.~(5) and Eq.~(6)
\citep{zhou2026robust}. We report two distinct versions: (i) its robust increment accumulated on the
matched CADS acquisition schedule, labeled ``Zhou--Liu evidence,'' and (ii) a literal implementation
of published Algorithm~2, including its prescribed initial round-robin of
$\lceil(4/\varepsilon)\log(4K/q^2)\rceil=588$ pulls per arm, or $3{,}528{,}000$ reads at $K=6000$.
The first is not mislabeled as their sampling algorithm. At our matched budget of $45{,}899$ reads that
round-robin is $77$-fold incomplete. We therefore treat both the literal Algorithm~2 and the
terminal-evidence diagnostic, which applies e-BH regardless of the schedule, as budget-compatibility
checks rather than evidence of general dominance.

\subsection{Certificate Results}

We first test CADS at the frozen operating point and across the registered grid, then expose
the power cost and the assumptions on which the guarantee depends.

\begin{figure}[!t]
\centering
\includegraphics[width=\textwidth]{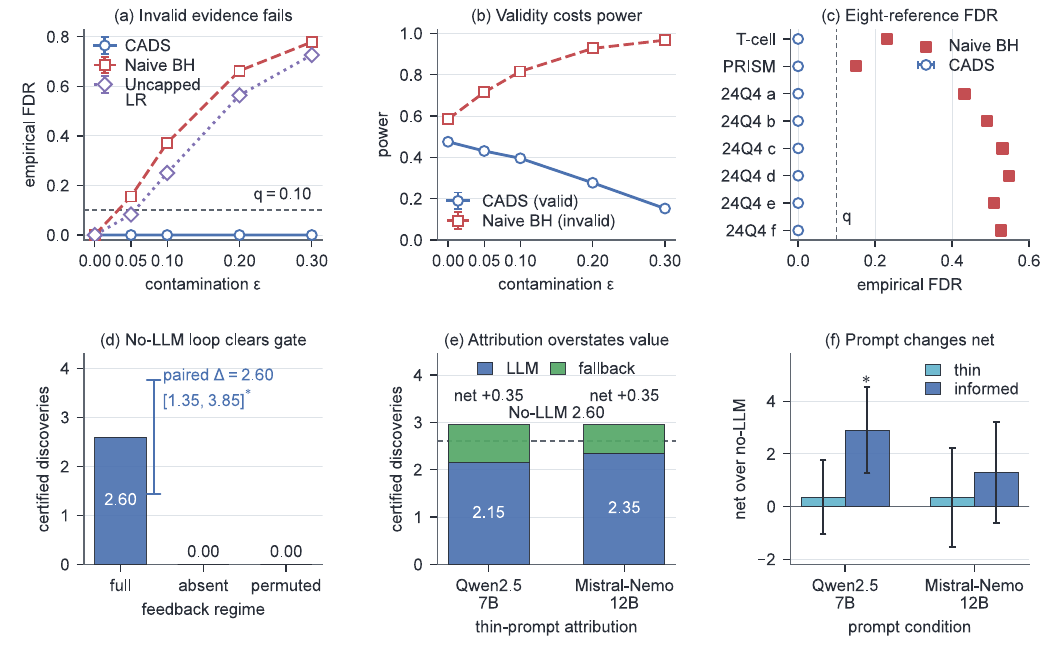}
\caption{CADS validity and proposer measurement. The top row uses the $m{=}6{,}000$
sweep registry ($q{=}0.10$; means and 95\% Monte Carlo intervals over 1000 paired campaigns in
(a)--(b) and 200 in (c)). \textbf{(a)}~Naive BH and uncapped LR e-BH exceed the target as contamination
grows, whereas CADS stays near zero. \textbf{(b)}~CADS pays for validity in power; naive BH's higher
apparent power is invalid. \textbf{(c)}~CADS remains below $q$ across one T-cell Perturb-seq, one
PRISM, and six DepMap 24Q4 references, whereas naive BH does not (identifiers in Appendix~D).
The bottom row uses a separate $m{=}18{,}159$ registry with 20 paired seeds and is never compared
across the registry boundary. \textbf{(d)}~The no-LLM loop clears the pre-specified feedback gate
($2.60$ vs.\ $0.00$, paired $[1.35,3.85]$). \textbf{(e)}~Origin-attributed discoveries
($2.95$) overstate the net over the identical no-LLM loop ($+0.35$) because LLM picks displace
fallback discoveries. \textbf{(f)}~An informed prompt moves Qwen from $+0.35$ to $+2.90$
($^{\ast}$ excludes zero); Mistral-Nemo moves in the same direction without excluding zero.
The $50$-seed replication in Appendix~G does not reproduce this initial magnitude. Zhou--Liu
evidence is omitted from (a)--(b) because it makes no discovery and is nonzero only for PRISM in (c)
(power $0.127$).}
\label{fig:results}
\end{figure}

\begin{table}[t]
\centering
\small
\begin{tabular}{lrrrrrr}
\toprule
Method & FDR & MC-SE & Power & MC-SE & True & False \\
\midrule
CADS & $1.33\times10^{-4}$ & $2.28\times10^{-5}$ & 0.396 & $3.76\times10^{-4}$ & 245.7 & 0.033 \\
Naive BH & 0.372 & $4.28\times10^{-4}$ & 0.816 & $4.40\times10^{-4}$ & 506.9 & 300.8 \\
Uncapped LR e-BH & 0.251 & $4.91\times10^{-4}$ & 0.711 & $4.92\times10^{-4}$ & 441.8 & 147.9 \\
Zhou-Liu evidence & 0.000 & 0.000 & 0.000 & 0.000 & 0.000 & 0.000 \\
\bottomrule
\end{tabular}
\caption{Primary operating point ($\varepsilon{=}0.10$, $q{=}0.10$, 1000 paired seeds). FDR is mean
FDP; MC-SE is across campaigns. True/false are mean terminal discovery counts; Appendix~D treats
all-zero uncertainty.}
\label{tab:primary}
 \end{table}

\paragraph{CADS controls FDR at the frozen operating point.}
At the primary point (Table~\ref{tab:primary}), CADS returns $245.7$ true discoveries and $0.033$ false
discoveries on average,
with power $0.396$. Naive BH has higher apparent power ($0.816$) but averages $300.8$ false
declarations, yielding FDR $0.372$. Uncapped LR e-BH yields FDR $0.251$. Across the entire
$5\times3$ $(\varepsilon,q)$ grid, CADS's largest empirical FDR is $6.73\times10^{-4}$ (MC-SE
$5.55\times10^{-5}$, at $\varepsilon{=}q{=}0.20$), while power decreases from $0.476$ at
$\varepsilon=0$ to $0.152$ at $\varepsilon=0.30$ for $q=0.10$ (Fig.~\ref{fig:results}(a,b)). This is a
worst-case-validity cost, not evidence that contamination improves discovery.

\paragraph{Control persists across reference settings.}
Across the same eight settings (one T-cell Perturb-seq, one PRISM, and six DepMap 24Q4 lines),
CADS FDR ranges from $1.21\times10^{-4}$ to $4.51\times10^{-4}$ and power from $0.189$ to $0.741$.
Naive BH FDR ranges from $0.151$ to $0.548$; uncapped LR e-BH ranges from $0.088$ to $0.402$
(Fig.~\ref{fig:results}(c); the per-reference table and interval dot plot are in Appendix~D). The
result is cross-reference robustness under one frozen simulator, not eight independent wet-lab
validations.

\paragraph{Validity has a power cost.}
At the matched budget, both Zhou--Liu checks make zero discoveries on $200$ seeds. This is budget
incompatibility in our regime, not general dominance of CADS. CADS itself loses power as the
contamination budget grows, as Figure~\ref{fig:results}(b) and the primary-grid results above show.

\paragraph{The guarantee depends on the contamination model.}
The frozen protocol injects one member of \eqref{eq:nullfamily}, and a kind one.
Appendix~F runs six adversarial alternatives at the same marginal rate,
including one fired only on candidates the agent chose to re-read; the
certificate holds for all of them, at worst FDR $0.0095$. A one-coin-per-campaign
all-or-nothing channel keeps the marginal rate but breaks the conditional one,
raising FDR two orders of magnitude to $0.0474$, still below $q$ here but with
the bound gone. At $\bar\varepsilon=\varepsilon$, a null mean $0.05$ above the
assumed boundary, a null standard deviation underestimated by ten percent, or
a same-mean, same-variance null with kurtosis $9.7$ breaks the budget
(Appendix~E). Setting $\bar\varepsilon=0.20$ absorbs a $0.10$ location shift or
a $1.5\times$ scale error; beyond the boundary, empirical FDR is $0.0126$ at a
$0.30$ shift, but we claim the guarantee rather than this realized margin.

At our audit channel ($\tau=1$, $s=0.5$) with $\Delta=2$, which contains the
simulator's $\mathrm{Uniform}(2,4)$ corruption,
Proposition~\ref{prop:kappa} gives $\kappa=0.977$, whereas the realized
sensitivity of that particular $h$ is $0.998$; declaring the class rather than
the law costs $2.1\%$ of sensitivity. The same detector has clean false-alarm
rate $\alpha_R=0.023$, which we simulate rather than idealize away: clean flags
inflate $p$ and make $\widehat\varepsilon_t$ more conservative.

\paragraph{Conditional audits recover only part of the lost power.}
Under the assumptions of Theorem~\ref{thm:audit} ($\delta=0.01$, $\rho=0$, true contamination $0.10$),
blinded controls given only the derived $\kappa=0.977$ recover $57\%$ of the blind-to-oracle power gap
on DepMap ($|A|=1000$: power $0.276$ versus oracle $0.390$ and blind $0.126$; coverage $1.000$, FDR
$5.68\times10^{-6}$), with audit reads charged on top of the test budget rather than taken from it.
Short audits can be worse than not auditing: at $|A|=50$ power falls to $0.023$. Appendix~D gives the
audit-size and $(\kappa,\rho)$ grids, a boundary-assumption stress test rather than validation of
$\kappa$ or $\rho$.

\section{Measuring LLM Proposer Contribution}
\label{sec:measurement}

The certificate of Sections~\ref{sec:method}--\ref{sec:experiments} is our measuring instrument.
We now use it to ask what an open-weight LLM proposer contributes to a discovery loop, a question a raw
hit count cannot answer because the count mixes the proposer's contribution with the loop's and carries
no error guarantee. The paired targets are defined in
Section~\ref{sec:measurement-targets}; this section specifies and evaluates the
corresponding loops.

\subsection{Paired Evaluation Protocol}

Oracle labels withheld from every proposer score the paired net contribution
post hoc over seeds.
The measurement uses the full DepMap 26Q1 ACH-000446 registry ($m{=}18{,}159$; $1{,}822$
reference hits; the $6{,}000$-gene validity sweeps above are never compared across that boundary).
Each campaign runs $12$ rounds of batch $40$ ($480$
reads over $260$ observed candidates) at $\varepsilon{=}q{=}0.10$, on paired seeds $0$--$19$. A proposer
sees only identifiers and the protocol's summary of its own past readouts, which here is the partition
of the candidates it has seen into high- and low-reading identifier lists, never their values. When the
LLM returns too few valid unseen symbols, a seed-paired label-free random draw fills the remaining
slots, and a non-LLM policy re-observes a fixed fraction of the observed set each round. We run
Qwen2.5-7B-Instruct, Mistral-Nemo-Instruct-2407, Meta-Llama-3.1-8B-Instruct, and the comparison the
field omits: an identical \emph{no-LLM loop} whose proposer returns nothing.

\subsection{Measurement Results}

The certificate holds throughout: empirical FDR is $0.000$ in all eighteen model $\times$ prompt
$\times$ feedback cells, whereas naive BH on the same streams runs $0.129$--$0.628$.
The paired results then ask whether feedback, proposal origin, or prompt state identifies the
proposer's incremental contribution.

\paragraph{Feedback alone does not identify LLM contribution.}
The standard ablation asks whether full feedback beats absent (none) and permuted (readout-shuffled)
feedback under a pre-specified paired gate. It does, but it also does so with \emph{no LLM in the loop}:
the no-LLM loop certifies $2.60$ true discoveries under full feedback and $0.00$ under both absent and
permuted, with paired $95\%$ $t$ interval $[1.35,3.85]$ against each (Fig.~\ref{fig:results}(d)). Its
non-LLM re-selection policy responds to feedback on its own, re-observing candidates that read high
whether or not an LLM chose them, so the ablation does not by itself measure the LLM.

\paragraph{Origin counts overstate net contribution.}
The direct alternative attributes each certified discovery to the mechanism that first proposed it,
since re-observation deepens evidence but never originates. Under the thin prompt Qwen certifies $2.95$
discoveries per campaign, $2.15$ LLM-originated and $0.80$ from the fallback, which read alone says the
LLM originates the majority. But the LLM fills only $58$ of the $260$ fresh slots, so the fallback fills
$202$ and finds $0.80$; remove the LLM and it fills all $260$ and finds $2.60$. The $2.15$ originated
discoveries therefore \emph{displace} $1.80$ the fallback would have made, leaving a net of $+0.35$
$[-1.05,1.75]$ (Fig.~\ref{fig:results}(e)). Only the head-to-head control gives the net.

\paragraph{Prompt information changes delivery, not a stable net magnitude.}
The thin prompt never tells the model which candidates it has already tested, so it is scored a
duplicate for re-proposing a gene it was shown reads high. An \emph{informed} prompt that lists the
tested candidates and asks for unseen ones changes only the prompt. Under it Qwen stops repeating and
the measured contribution moves with it (Fig.~\ref{fig:results}(f); Appendix~G): on these $20$ seeds the net over
the no-LLM loop goes from $+0.35$ to $+2.90$ $[1.26,4.54]$ and the prompt effect is $+2.55$
$[1.40,3.70]$, both excluding zero. Mistral-Nemo moves the same way without excluding zero and
Llama-3.1-8B already nets $+2.80$ thin, with a further $+0.25$ whose interval includes zero, so
the pattern is model-dependent. Qwen carried the only interval above zero, so we designate it the primary contrast and
the rest exploratory; that designation follows the estimate rather than preceding it, and the starred
intervals are uncorrected. A paired percentile bootstrap ($10{,}000$ resamples) agrees (net
$[1.40,4.35]$, prompt effect $[1.50,3.60]$).
In all three models the informed prompt raises delivered-candidate counts. It fills $176$ of $260$ fresh slots against thin's $58$,
while per-slot reference-hit precision differs by $+0.010$ $[-0.037,0.057]$, an interval too wide to
call the two prompts' candidate quality equal. Holding every proposer to an identical per-round quota
on $50$ fresh seeds (Appendix~G) leaves that delivery gap as the only effect common to all three models
(thin misses $40.7$, $26.0$ and $13.1$ of the $55$ late slots). No model's prompt effect then excludes
zero (Qwen $+0.74$ $[-0.12,1.66]$, Mistral-Nemo $+0.16$, Llama $-0.64$; Holm $p\ge0.34$) and Qwen's
natural effect attenuates to $+1.10$ $[-0.10,2.28]$. That design equalises the quota requested, not
the delivery achieved, so we report the delivery difference as the one effect common to all three
models and leave the magnitude unresolved.

\paragraph{A fixed non-LLM prior remains strongest.}
A competent non-LLM baseline dwarfs these gains: a truth-isolated pan-essentiality prior (mean gene
effect across the DepMap 24Q4 lines other than the target, fixed before any readout) certifies $12.65$
at empirical FDR $0$, a net $+10.05$ $[8.69,11.41]$ (Appendix~G).

\section{Limitations}
\label{sec:limitations}

CADS assumes a registered candidate family, a terminal decision, a correctly specified one-sided
null, and a valid contamination bound. Misspecifying the working alternative reduces power;
misspecifying the null can invalidate control. Reference labels expose FDP only in simulation, while
Audit-CADS additionally needs randomized blinded controls, externally justified flag sensitivity,
and bounded test--audit mismatch. The theory permits optional stopping, but the experiments use fixed
rounds.

The empirical scope is three open-weight models in one semi-synthetic, DepMap-backed loop. No model's
prompt effect excludes zero on $50$ fresh seeds, so the evidence does not establish a stable model
ranking. The qualitative conclusions are concurrent with \citet{gupta2025bayesopt}, and pretraining
leakage remains untested. FDR control applies to each loop's terminal set, not to between-loop or
prompt-effect contrasts. Those contrasts are descriptive; the Qwen contrast was selected after
inspecting the estimates and is uncorrected. Registry, fallback, prompt, or delivery changes alter
the estimand.

\section{Conclusion}
\label{sec:conclusion}

CADS measures proposer contribution by applying the same certified decision rule to full and
proposer-removed loops. In the tested screen, feedback gates and origin counts overcredit the LLM
because re-selection and fallback can also produce discoveries, while a fixed non-LLM prior remains
strongest. Prompt history initially increased Qwen's net contribution, but this effect did not
persist in the $50$-seed matched-quota replication. The comparison remains configuration-specific:
its guarantee certifies each terminal discovery set, not the difference between loops, and changes
to the registry, fallback, prompt, or delivered-candidate count alter the measured contrast.
Sequential screens should therefore report certified net gain over matched no-LLM baselines,
distinguish delivery from candidate quality, and treat cross-loop contrasts as measurements rather
than FDR-controlled discoveries.

\noindent\textbf{AI-use disclosure.} AI assisted editing, engineering, and
review; authors verified all content and remain responsible.
\clearpage
\begin{center}
{\Large\bfseries Supplementary Material\par}
\vspace{0.45em}
{\normalsize CADS: Measuring the Incremental Contribution of LLM Proposers in Sequential Biological Screens\par}
\end{center}
\vspace{0.5em}

\noindent\textbf{Supplement map.}
Appendix~A gives complete proofs. Appendices~B--C specify implementation,
truth isolation, and the frozen protocol. Appendices~D--G report corrected
results, null misspecification, excluded contamination channels, and proposer
contribution analyses. Appendix~H documents reproducibility and artifact
integrity.

\appendix
\setcounter{theorem}{0}
\setcounter{proposition}{0}
\setcounter{lemma}{0}
\setcounter{figure}{0}
\setcounter{table}{0}
\renewcommand{\thetheorem}{A.\arabic{theorem}}
\renewcommand{\theproposition}{A.\arabic{proposition}}
\renewcommand{\thelemma}{A.\arabic{lemma}}
\numberwithin{equation}{section}
\renewcommand{\theHtheorem}{supp.\arabic{theorem}}
\renewcommand{\theHproposition}{supp.\arabic{proposition}}
\renewcommand{\theHlemma}{supp.\arabic{lemma}}
\renewcommand{\thefigure}{S\arabic{figure}}
\renewcommand{\thetable}{S\arabic{table}}
\section{Complete Proofs}
\label{sec:proofs}

We use the notation of the main paper. The registered family is
$\mathcal H=\{H_i:i\in[m]\}$, and $\mathcal N\subseteq[m]$ denotes its true nulls.
Conditional on the past, a null read has a law
\begin{equation}
Q=(1-\varepsilon_0)\mathcal N(\mu,s^2)+\varepsilon_0 H,
\qquad \mu\le0,\quad \varepsilon_0\le\bar\varepsilon,
\label{eq:supp-null}
\end{equation}
where $H$ is arbitrary. The cap is constructed for the conservative rate
$\bar\varepsilon$. The working alternative is $\mathcal N(\theta,s^2)$ for fixed
$\theta>0$.

\subsection{Capped likelihood-ratio increment}

Let $f_0$ and $f_1$ be the densities of $\mathcal N(0,s^2)$ and
$\mathcal N(\theta,s^2)$, respectively, and write
\begin{equation}
L(y)=\frac{f_1(y)}{f_0(y)}
=\exp\!\left\{\frac{\theta y}{s^2}-\frac{\theta^2}{2s^2}\right\}.
\label{eq:supp-lr}
\end{equation}
For $0<\bar\varepsilon<1$, define
\begin{equation}
G_{\bar\varepsilon}(C)
=(1-\bar\varepsilon)\mathbb E_{0}[\min\{L,C\}]
 +\bar\varepsilon C.
\label{eq:supp-budget}
\end{equation}

\begin{lemma}[Cap and Gaussian closed form]
\label{lem:cap}
There is a unique $C^\ast>1$ satisfying
$G_{\bar\varepsilon}(C^\ast)=1$. If
$t(C)=s^2\log(C)/\theta+\theta/2$, then
\begin{equation}
\mathbb E_0[\min\{L,C\}]
=\Phi\!\left(\frac{t(C)-\theta}{s}\right)
+C\left\{1-\Phi\!\left(\frac{t(C)}{s}\right)\right\}.
\label{eq:supp-closed}
\end{equation}
\end{lemma}
\begin{proof}
The likelihood ratio is increasing and $L(y)\le C$ exactly when $y\le t(C)$.
Changing measure on this event gives
\begin{align*}
\mathbb E_0[L(Y)\mathbf 1\{Y\le t(C)\}]
&=\int_{-\infty}^{t(C)}f_1(y)\,dy\\
&=\Phi\!\left(\frac{t(C)-\theta}{s}\right).
\end{align*}
The complementary contribution is
$C\Pr_0\{Y>t(C)\}$, proving \eqref{eq:supp-closed}.
The map $G_{\bar\varepsilon}$ is continuous and strictly increasing. Because
$\theta>0$, $\mathbb E_0[\min\{L,1\}]<\mathbb E_0[L]=1$, so
$G_{\bar\varepsilon}(1)<1$. Also
$G_{\bar\varepsilon}(C)\to\infty$ as $C\to\infty$ because
$\bar\varepsilon C\to\infty$. The intermediate value theorem and strict
monotonicity give the unique root.
\end{proof}

\begin{theorem}[Composite-null robust increment]
\label{thm:supp-increment}
Let $e_C(y)=\min\{L(y),C^\ast\}$, with $C^\ast$ from
Lemma~\ref{lem:cap}. For every law in \eqref{eq:supp-null},
\[
\mathbb E_Q[e_C(Y)]\le1.
\]
For $\bar\varepsilon=0$, the same statement holds with $e_C=L$ and no
contamination; for $\bar\varepsilon=1$, it holds with $e_C=\min\{L,1\}$.
\end{theorem}
\begin{proof}
Since $L$ is increasing, $e_C$ is increasing. For $\mu\le0$,
$\mathcal N(\mu,s^2)$ is stochastically dominated by
$\mathcal N(0,s^2)$, and hence
\begin{equation}
\mathbb E_{\mu}[e_C(Y)]\le \mathbb E_0[e_C(Y)].
\label{eq:supp-boundary}
\end{equation}
The arbitrary contamination component contributes at most $C^\ast$ because
$e_C\le C^\ast$ pointwise. Therefore
\[
\mathbb E_Q[e_C]
\le (1-\varepsilon_0)\mathbb E_0[e_C]+\varepsilon_0 C^\ast.
\]
The right side is nondecreasing in $\varepsilon_0$, since
$C^\ast\ge\mathbb E_0[e_C]$. It is thus at most its value at
$\bar\varepsilon$, which equals one by \eqref{eq:supp-budget}. The endpoint
cases follow directly from $\mathbb E_\mu[L]\le\mathbb E_0[L]=1$ and
$e_C\le1$, respectively.
\end{proof}

\begin{proposition}[Conservative contamination ratings are safe]
\label{prop:overrate}
Under all hypotheses of Theorem~\ref{thm:supp-increment}, an increment
constructed with $\bar\varepsilon$ remains valid whenever the actual
conditional contamination probability $\varepsilon_0$ is no greater than
$\bar\varepsilon$. Under-rating is not covered, and the rate bound alone does
not suffice: the clean-kernel condition is what makes $a=\mathbb E_0[e_C]$
finite and at most $C^\ast$.
\end{proposition}
\begin{proof}
This is the monotonicity step in the proof of
Theorem~\ref{thm:supp-increment}: for
$a=\mathbb E_0[e_C]\le C^\ast$, the worst-case expectation
$a+\varepsilon_0(C^\ast-a)$ is nondecreasing in $\varepsilon_0$.
\end{proof}

\subsection{Adaptive accumulation and complete-family e-BH}

Let $\mathcal F_t$ include all proposals, assay readouts, and internal
randomness available through round $t$. Let $A_{i,t}\in\{0,1\}$ indicate that
candidate $i$ is assayed at $t$; it is $\mathcal F_{t-1}$-measurable. Define
\begin{equation}
\begin{split}
M_{i,0}&=1,\\
M_{i,t}&=M_{i,t-1}
\{A_{i,t}e_C(Y_{i,t})+(1-A_{i,t})\}.
\end{split}
\label{eq:supp-process}
\end{equation}

\begin{lemma}[Predictable allocation preserves the e-process]
\label{lem:eprocess}
If the conditional null law satisfies \eqref{eq:supp-null} whenever a null
candidate is assayed, then $(M_{i,t})_{t\ge0}$ is a nonnegative
supermartingale for every $i\in\mathcal N$. At every possibly unbounded
stopping time $\tau$, setting $M_{i,\tau}:=\liminf_{n\to\infty}M_{i,\tau\wedge n}$ on
$\{\tau=\infty\}$, where nonnegative supermartingale convergence makes the limit exist
almost surely, $E_i=M_{i,\tau}$ is an e-value: $\mathbb E[E_i]\le1$.
\end{lemma}
\begin{proof}
For a null $i$, predictability of $A_{i,t}$ and
Theorem~\ref{thm:supp-increment} give
\[
\mathbb E[M_{i,t}\mid\mathcal F_{t-1}]
\le M_{i,t-1}.
\]
Thus the process is a nonnegative supermartingale from one. Optional stopping
applies to the bounded time $\tau\wedge n$, giving $\mathbb E[M_{i,\tau\wedge n}]\le1$
for every $n$; since $M_{i,\tau\wedge n}\to M_{i,\tau}$ almost surely with the
convention above, Fatou's lemma gives $\mathbb E[M_{i,\tau}]\le1$. Every campaign we
report stops at a fixed round count, so $\tau$ is bounded there and this step is a
statement of generality rather than one the experiments exercise.
No restriction is placed on how the proposer uses past readouts.
\end{proof}

At the terminal time, CADS forms a vector of length exactly $m$. An
unobserved candidate has $M_{i,\tau}=1$. Write the sorted evidence as
$E_{(1)}\ge\cdots\ge E_{(m)}$, and set
\[
r=\max\{k:E_{(k)}\ge m/(qk)\},
\]
with $r=0$ if the set is empty.

\begin{theorem}[Registered-family FDR control]
\label{thm:supp-fdr}
Let $E_1,\dots,E_m\ge0$ with $\mathbb E[E_i]\le1$ for every true null
$i\in\mathcal N$, as Lemma~\ref{lem:eprocess} supplies. Then the top $r$
candidates returned by complete-family e-BH satisfy
$\operatorname{FDR}\le q|\mathcal N|/m\le q$, under arbitrary dependence
among candidates. The null e-value hypothesis is not decorative: a single
deterministic $E_1=1/q$ on a true null is rejected and gives
$\operatorname{FDR}=1$.
\end{theorem}
\begin{proof}
If $r>0$, every selected candidate satisfies $E_i\ge m/(qr)$, so
\[
\mathbf 1\{i\in R\}\le\frac{qr}{m}E_i.
\]
Summing this inequality over true nulls and dividing by $r$ gives the
pointwise relation
\[
\operatorname{FDP}(R)
\le\frac q m\sum_{i\in\mathcal N}E_i.
\]
The inequality is trivial when $r=0$. Taking expectations and applying
Lemma~\ref{lem:eprocess} proves the claim. This argument is the e-BH step of
\citet{xu2021bandit,wang2022fdr}; it uses no independence or positive
dependence assumption.
\end{proof}

The denominator $m$ is load-bearing. Replacing it by the random number of
observed candidates changes the rejection threshold after outcome-dependent
acquisition and is not justified by Theorem~\ref{thm:supp-fdr}.

\subsection{Conditional Audit-CADS}
\label{sec:supp-audit-proof}

Let $B_1,B_2,\ldots$ be blinded control flags. Conditional on fixed channel
parameters, assume they are i.i.d.\ Bernoulli with mean $p$ and independent of
test readouts. With a Beta$(a,b)$ mixing density, $a,b>0$, define
\begin{equation}
\mathcal M_n(p)=
\frac{\mathrm B(S_n+a,n-S_n+b)}
     {\mathrm B(a,b)\,p^{S_n}(1-p)^{n-S_n}},
\qquad S_n=\sum_{j=1}^n B_j.
\label{eq:supp-mixture}
\end{equation}
For $p\in(0,1)$ this is, at the true $p$, the mixture over Bernoulli
alternatives of likelihood ratios against $p$, and is therefore a nonnegative
martingale started at one. At the endpoints the likelihood ratio degenerates and
$\mathcal M_n$ is a nonnegative supermartingale started at one instead: at $p=0$
every flag is zero almost surely and $\mathcal M_n(0)=\mathrm
B(a,n+b)/\mathrm B(a,b)<1$ for $n\ge1$, and symmetrically at $p=1$. Ville's
inequality needs only the supermartingale property, so it applies throughout;
we state the martingale case because it is the one the interior uses. Ville
yields
\[
\Pr_p\{\exists n:\mathcal M_n(p)\ge1/\delta\}\le\delta.
\]
The implementation uses $a=b=1/2$, inverts
$\mathcal M_n(p)<1/\delta$, takes its upper endpoint, and then takes the
running intersection. Denote the resulting nonincreasing upper path by
$U_n$. Consequently,
\begin{equation}
\Pr\{p\le U_n\ \text{for all }n\}\ge1-\delta.
\label{eq:supp-cs}
\end{equation}

Assume each harmful audit corruption is flagged with conditional probability
at least $\kappa>0$; clean false flags are allowed. If
$\varepsilon_{\rm audit}$ is the audit corruption rate, then
$p\ge\kappa\varepsilon_{\rm audit}$. Also assume the pre-registered mismatch
condition
\begin{equation}
\begin{aligned}
\Pr\!\left(C_{i,t}=1\mid\mathcal F_{t-1},\mathcal A_\infty,A_{i,t}=1\right)
&\le\varepsilon_{\rm test}\\
&\le\varepsilon_{\rm audit}+\rho
\end{aligned}
\label{eq:supp-mismatch}
\end{equation}
almost surely for every null $i$ and round $t$,
for pre-registered $\rho\ge0$, where $C_{i,t}$ indicates corruption of read
$(i,t)$, $A_{i,t}$ that the read is taken, $\mathcal F_{t-1}$ is the test
history and $\mathcal A_\infty$ the whole audit sequence. The conditioning is
not decoration. A marginal rate is strictly weaker and does not suffice: a
channel that flips one coin once and then corrupts either every read or none
has marginal rate $\varepsilon_{\rm test}$ and satisfies any scalar version of
this condition, yet on its corrupted branch the null's increments all saturate
the cap and three of them already exceed $1/(q-\delta)$ at
$\varepsilon_{\rm test}=0.2$, $q=0.10$, $\delta=0.01$, so it is declared and FDR
approaches $0.2$. Condition~\eqref{eq:supp-mismatch} excludes exactly that, and
no audit can verify it; it is an assumption on the laboratory's channel.
Before test round $t$, CADS uses only earlier controls:
\begin{equation}
\widehat\varepsilon_t
=\min\{1,U_{t-1}/\kappa+\rho\}.
\label{eq:supp-auditcap}
\end{equation}

\begin{theorem}[Conditional Audit-CADS]
\label{thm:supp-audit}
Let $0<\delta<q<1$ and $\kappa>0$. Assume the hypotheses of Theorem~\ref{thm:supp-fdr}:
a family registered before the campaign, the declared clean-null family
\eqref{eq:supp-null} conditional on the past, predictable allocation, and e-BH on the
complete family at the terminal stopping time. Assume in addition the flag model,
independence of the entire audit flag sequence from the complete collection of test
potential outcomes and from all test-side randomization, and
\eqref{eq:supp-mismatch}. Then using the predictable caps
\eqref{eq:supp-auditcap} and complete-family e-BH at level $q-\delta$
controls unconditional FDR at $q$.
\end{theorem}
\begin{proof}
Let $\mathcal G=\{p\le U_n\text{ for all }n\}$. On $\mathcal G$,
\[
\varepsilon_{\rm test}
\le\varepsilon_{\rm audit}+\rho
\le p/\kappa+\rho
\le U_{t-1}/\kappa+\rho
\]
before every test round, where $p\ge\kappa\varepsilon_{\rm audit}$ because a
corrupted control flags with probability at least $\kappa$ and clean controls
only add to $p$. The quantity so bounded is the \emph{conditional} corruption
probability of \eqref{eq:supp-mismatch}, not a marginal rate; this is what the
next step needs, and a marginal rate would not supply it. Proposition~\ref{prop:overrate}
then makes every test increment valid for the realized audit path, since its
hypothesis is exactly that the conditional null kernel lies in the ball the cap
was solved for. Independence is used here, and at its full strength: conditioning
on an audit path in $\mathcal G$ conditions on the entire flag sequence,
including flags that arrive later, so we need the audit sequence to be
independent of the test potential outcomes and test-side randomization, not
merely of the readouts that happen to be observed, since which readouts are
observed is itself audit-dependent through the caps. Under that assumption the
test-stream null expectations retain the bound required by
Lemma~\ref{lem:eprocess}.
Thus Theorem~\ref{thm:supp-fdr}, conditionally on that path, bounds FDR by
$q-\delta$. Since FDP is at most one,
\[
\mathbb E[\operatorname{FDP}]
\le(q-\delta)\Pr(\mathcal G)+\Pr(\mathcal G^c)
\le(q-\delta)+\delta=q.
\]
The single failure event is global; no factor of $m$ is introduced.
\end{proof}

The conclusion is conditional on identifiable links between controls and test
readouts. It does not cover an unknown $\kappa$, mismatch exceeding $\rho$,
target-specific nonstationarity, or corruptions that evade the flagger.

\section{Implementation and Truth-Isolation Contract}
\label{sec:implementation}

\paragraph{Four separated objects.}
The candidate registry contains identifiers only. The proposer view contains
identifiers and past structured numerical readouts. The certificate contains
the evidence code and never reads reference effects or labels. A private
evaluation object contains the public-reference effect and hit label and is
used only after terminal declarations to compute FDP and power.

\paragraph{Fallback behavior.}
The LLM adapter permits at most two generation attempts. It parses identifiers,
removes invalid and duplicate entries, and fills any shortfall by seeded
uniform sampling from eligible unseen identifiers. The fallback does not rank
by the private effect or label. Raw text, parsed identifiers, invalid entries,
duplicates, retry count, and fallback count are retained round by round.

\paragraph{Complete family.}
Both BH diagnostics and e-BH receive a vector of length $m$. Unobserved
hypotheses receive $p=1$ and $E=1$. Index bounds, duplicate indices,
non-finite evidence, negative evidence, and wrong-dimensional inputs raise
errors.

\paragraph{Numerical safeguards.}
The GP-UCB implementation uses a Cholesky solve and explicit finite checks.
Log evidence is clipped only when converted for a terminal floating-point
comparison; the per-read analytical cap is computed by a bracketed root solve.
A real GP-UCB rerun that emitted platform-level numerical warnings was not
used for any reportable result. The corrected reported grids use the frozen
label-free random priority plus outcome-dependent re-observation.

\paragraph{Statistical contract tests.}
The test suite includes cap-budget equality, arbitrary-contamination upper
bounds, conservative-cap validity, the Gaussian composite-null boundary,
predictability of audit caps, time-uniform audit coverage, full-family
accounting, truth-isolation/fallback checks, Zhou--Liu formula checks, and
non-finite GP-state rejection. The composite-null test evaluates several
$\mu<0$ values and confirms that $\mu=0$ is the least-favorable Gaussian
boundary.

\section{Frozen Experimental Protocol}
\label{sec:protocol}

\paragraph{Reference backbones.}
The primary registry is DepMap 26Q1 ACH-000446. Its threshold-centered signal
is $u_i=-\mathrm{effect}_i-0.5$, and the reference alternative is $u_i>0$.
The cross-setting evaluation contains one T-cell Perturb-seq DOP1B/Rest
registry, one PRISM ACH-000827 registry using the public sensitivity threshold,
and six DepMap 24Q4 registries using the released dependency-probability label.
Each loader asserts that reference nulls have $u_i\le0$ and reference
alternatives have $u_i\ge0$. These quantities form the signal backbone and
private evaluation oracle; they are not proposal inputs.

\paragraph{Sequential simulator.}
A clean read is
$Y_{i,t}\sim\mathcal N(u_i,0.5^2)$. Independently for every read, with
probability $\varepsilon$ an upward value sampled uniformly from $[2,4]$ is
added. This is one difficult contamination distribution inside the larger
adversarial family covered by the theorem. The evidence working alternative is
$\theta=0.7$.

A seeded standard-normal priority orders the registered candidates without
using $u_i$ or the labels. Round zero observes the first 500 candidates.
For each of 15 continuation rounds, the simulator re-observes the top 70\% of
currently observed candidates by their empirical mean and then introduces up
to 500 fresh candidates. Terminal evidence is evaluated on the complete
registry. The primary registry uses 45,899 reads per campaign.

\paragraph{Methods.}
\emph{CADS} uses the capped LR and full-family e-BH.
\emph{Uncapped LR e-BH} removes the cap and is valid only at zero
contamination.
\emph{Naive BH} applies terminal one-sided Gaussian $p$-values as if the
outcome-dependent sample sizes were fixed; it is a failure diagnostic.
\emph{Zhou--Liu evidence} accumulates their Eq.~(5) robust increment on the
same CADS acquisition schedule and then applies full-family e-BH. It is not
called their sampling algorithm. A separate reproduction implements their
published Algorithm~2.

\paragraph{Seeds and uncertainty.}
The primary $5\times3$ $(\varepsilon,q)$ grid uses the same 1000 seeds in
every cell. Each cross setting and the strict Zhou--Liu comparison use 200
seeds. FDR is the mean campaign FDP, including zero when a campaign has no
declaration. Power is the mean fraction of reference alternatives declared.
MC-SE is the sample standard deviation across campaigns divided by the square
root of their number. Figures use descriptive mean $\pm1.96$ MC-SE error
bars, clipped to $[0,1]$ for bounded metrics; JSON fields ending in
\texttt{normal\_errorbar95} are not confidence intervals. Audit-CADS FDR also
stores a distribution-free one-sided $95\%$ upper bound. When every FDP is
zero, the normal bar collapses to a point; read it as zero observed FDP, not
as a confidence statement. For a
$[0,1]$-valued FDP with an all-zero sample of size $n$, the distribution-free
one-sided $95\%$ upper bound on $\mathbb E[\mathrm{FDP}]$ is $1-0.05^{1/n}$,
which is $0.0149$ at $n{=}200$ and $0.0030$ at $n{=}1000$. Appendix~G reports
that bound directly for the follow-up experiment.

\section{Corrected Results}
\label{sec:results}

\subsection{Complete primary grid}

\begin{table}[t]
\centering
\small
\begin{tabular}{rrrrrr}
\toprule
$q$ & $\varepsilon$ & CADS FDR (SE) & CADS power (SE) & Naive FDR & LR FDR \\
\midrule
0.05 & 0.00 & $2.84\times10^{-5}$ ($1.00\times10^{-5}$) & 0.458 ($3.35\times10^{-4}$) & 0.000 & 0.000 \\
0.05 & 0.05 & $6.31\times10^{-5}$ ($1.57\times10^{-5}$) & 0.405 ($3.55\times10^{-4}$) & 0.131 & 0.073 \\
0.05 & 0.10 & $5.57\times10^{-5}$ ($1.60\times10^{-5}$) & 0.348 ($3.99\times10^{-4}$) & 0.337 & 0.232 \\
0.05 & 0.20 & $5.61\times10^{-5}$ ($1.98\times10^{-5}$) & 0.226 ($5.36\times10^{-4}$) & 0.640 & 0.548 \\
0.05 & 0.30 & $1.79\times10^{-5}$ ($1.79\times10^{-5}$) & 0.093 ($6.63\times10^{-4}$) & 0.768 & 0.718 \\
\midrule
0.10 & 0.00 & $6.16\times10^{-5}$ ($1.44\times10^{-5}$) & 0.476 ($3.32\times10^{-4}$) & 0.001 & 0.000 \\
0.10 & 0.05 & $1.60\times10^{-4}$ ($2.44\times10^{-5}$) & 0.431 ($3.54\times10^{-4}$) & 0.156 & 0.083 \\
0.10 & 0.10 & $1.33\times10^{-4}$ ($2.28\times10^{-5}$) & 0.396 ($3.76\times10^{-4}$) & 0.372 & 0.251 \\
0.10 & 0.20 & $1.77\times10^{-4}$ ($3.14\times10^{-5}$) & 0.277 ($6.63\times10^{-4}$) & 0.663 & 0.564 \\
0.10 & 0.30 & $9.12\times10^{-5}$ ($3.03\times10^{-5}$) & 0.152 ($5.27\times10^{-4}$) & 0.779 & 0.727 \\
\midrule
0.20 & 0.00 & $1.31\times10^{-4}$ ($2.03\times10^{-5}$) & 0.494 ($3.34\times10^{-4}$) & 0.002 & 0.000 \\
0.20 & 0.05 & $3.37\times10^{-4}$ ($3.42\times10^{-5}$) & 0.457 ($3.54\times10^{-4}$) & 0.186 & 0.094 \\
0.20 & 0.10 & $3.68\times10^{-4}$ ($3.69\times10^{-5}$) & 0.428 ($3.72\times10^{-4}$) & 0.412 & 0.270 \\
0.20 & 0.20 & $6.73\times10^{-4}$ ($5.55\times10^{-5}$) & 0.339 ($7.15\times10^{-4}$) & 0.688 & 0.580 \\
0.20 & 0.30 & $5.37\times10^{-4}$ ($6.18\times10^{-5}$) & 0.226 ($4.98\times10^{-4}$) & 0.791 & 0.734 \\
\bottomrule
\end{tabular}
\normalsize
\caption{Complete primary DepMap grid. CADS cells give mean (MC-SE)
over 1000 paired campaigns. Naive and uncapped-LR columns give mean FDR.
All decisions use the full $m=6000$ family.}
\label{tab:supp-grid}
\end{table}

At $\varepsilon=0$, CADS and uncapped LR e-BH coincide. As contamination
increases, the capped method becomes more conservative while the uncapped
e-process and naive $p$-value procedure cease to control FDR. The largest
observed CADS FDR in Table~\ref{tab:supp-grid} is
$6.73\times10^{-4}$ (MC-SE $5.55\times10^{-5}$), at
$\varepsilon=q=0.20$.

\subsection{Cross-reference settings}

\begin{figure}[t]
\centering
\includegraphics[width=\textwidth]{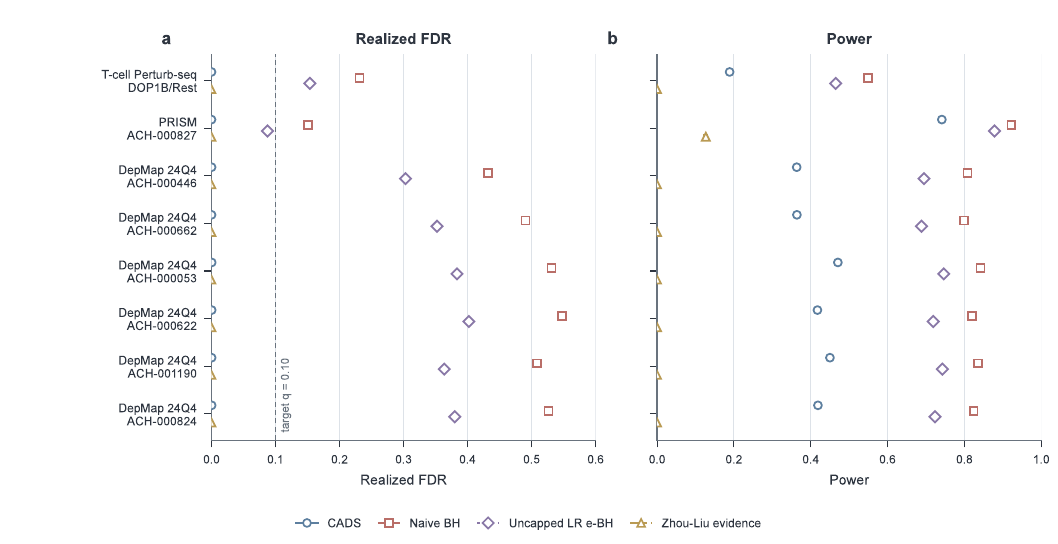}
\caption{Cross-setting operating points at
$\varepsilon=q=0.10$. Dots are seed means and bars are 95\% Monte Carlo
intervals over 200 paired campaigns. ZL is the matched-schedule
Zhou--Liu evidence diagnostic. The vertical dashed line in the FDR panel marks
$q=0.10$.}
\label{fig:supp-cross}
\end{figure}

\begin{table}[t]
\centering
\small
\resizebox{\textwidth}{!}{%
\begin{tabular}{lrrrrr}
\toprule
Reference dataset & Naive FDR & LR FDR & CADS FDR (SE) & CADS power (SE) & ZL power \\
\midrule
T-cell Perturb-seq DOP1B/Rest & 0.231 & 0.154 & $1.36\times10^{-4}$ ($6.03\times10^{-5}$) & 0.189 ($7.44\times10^{-4}$) & 0.000 \\
PRISM ACH-000827 & 0.151 & 0.088 & $1.21\times10^{-4}$ ($3.56\times10^{-5}$) & 0.741 ($5.86\times10^{-4}$) & 0.127 \\
DepMap 24Q4 ACH-000446 & 0.432 & 0.303 & $2.67\times10^{-4}$ ($7.39\times10^{-5}$) & 0.364 ($7.78\times10^{-4}$) & 0.000 \\
DepMap 24Q4 ACH-000662 & 0.490 & 0.352 & $2.04\times10^{-4}$ ($6.67\times10^{-5}$) & 0.364 ($9.18\times10^{-4}$) & 0.000 \\
DepMap 24Q4 ACH-000053 & 0.531 & 0.384 & $4.51\times10^{-4}$ ($7.77\times10^{-5}$) & 0.471 ($8.03\times10^{-4}$) & 0.000 \\
DepMap 24Q4 ACH-000622 & 0.548 & 0.402 & $4.31\times10^{-4}$ ($8.93\times10^{-5}$) & 0.417 ($9.75\times10^{-4}$) & 0.000 \\
DepMap 24Q4 ACH-001190 & 0.509 & 0.364 & $3.32\times10^{-4}$ ($8.19\times10^{-5}$) & 0.450 ($9.23\times10^{-4}$) & 0.000 \\
DepMap 24Q4 ACH-000824 & 0.527 & 0.380 & $3.87\times10^{-4}$ ($9.21\times10^{-5}$) & 0.419 ($9.37\times10^{-4}$) & 0.000 \\
\bottomrule
\end{tabular}%
}
\normalsize
\caption{Cross-setting results at $\varepsilon{=}q{=}0.10$ (200 paired
seeds per row). All procedures use the same complete-family decision code.
Values are seed means; Figure~\ref{fig:supp-cross} gives the dot plot with 95\%
intervals. ZL denotes the matched-schedule Zhou--Liu evidence diagnostic, not
Algorithm~2.}
\label{tab:supp-cross}
\end{table}

Figure~\ref{fig:supp-cross} and Table~\ref{tab:supp-cross} show the eight
settings reported in the main paper. CADS FDR is at most $4.51\times10^{-4}$ in
every setting, while power ranges from 0.189 to 0.741. These are cross-reference
simulations with one frozen readout model, not independent biological
replications.

\subsection{Closest robust bandit baseline}

The reproduced Zhou--Liu increment is
\begin{equation}
\begin{split}
e_{\rm ZL}(Y)
&=\frac{\exp\{\phi(\lambda(Y-\mu_0))\}}
{1+\lambda^2s^2/2+1.5\varepsilon},\\
\lambda&=\frac{\sqrt{\varepsilon}}{2s},
\end{split}
\label{eq:supp-zl}
\end{equation}
where $\phi$ is their four-piece Catoni influence function, bounded by
$\pm\log2$. Their Algorithm~2 begins with
\[
B(q,\varepsilon)
=\left\lceil\frac4\varepsilon
\log\!\left(\frac{4K}{q^2}\right)\right\rceil
\]
round-robin pulls per arm. For $K=6000$ and
$\varepsilon=q=0.10$, this is 588 pulls per arm.

\begin{table}[t]
\centering
\small
\begin{tabular}{lrrr}
\toprule
Zhou--Liu implementation & Required burn-in & Matched budget & Mean discoveries \\
\midrule
Published Algorithm 2 & 3,528,000 & 45,899 & 0.0 \\
Terminal Eq.~(5) evidence diagnostic & -- & 45,899 & 0.0 \\
\bottomrule
\end{tabular}
\normalsize
\caption{Strict Zhou--Liu budget comparison over 200 paired seeds.
The terminal evidence row is Eq.~\eqref{eq:supp-zl} plus e-BH at the matched
budget; it is not Algorithm~2.}
\label{tab:supp-zl}
\end{table}

Both rows in Table~\ref{tab:supp-zl} make zero discoveries. This result
documents a budget mismatch; it does not establish general superiority.
The original L\'evy-stable simulation description omits a distribution
parameter, and no bit-for-bit reproduction is claimed.

\subsection{Conditional Audit-CADS}

Figure~\ref{fig:supp-audit} shows the bound path as the audit grows, and
Table~\ref{tab:supp-audit} gives the corresponding numerical results.

\begin{figure}[t]
\centering
\includegraphics[width=\textwidth]{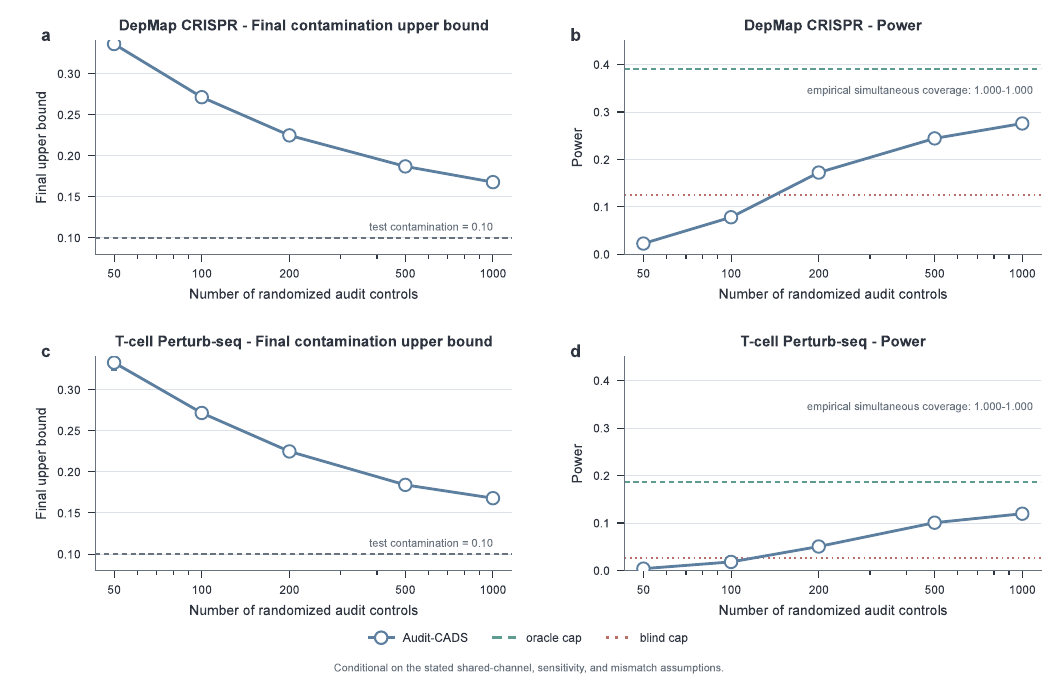}
\caption{Conditional Audit-CADS under its shared-channel
assumptions. Means and 95\% Monte Carlo intervals are shown as the number of
one-read randomized controls grows. The first test round uses $U_0=1$.
Dashed horizontal lines show the oracle-$\varepsilon$ and fixed blind-cap
power. Coverage is simultaneous across the cap path.}
\label{fig:supp-audit}
\end{figure}

\begin{table}[t]
\centering
\small
\resizebox{\textwidth}{!}{%
\begin{tabular}{lrrrrr}
\toprule
Reference dataset & Audit $|A|$ & Final bound & Coverage & FDR & Power \\
\midrule
DepMap CRISPR & 50 & 0.336 (0.002) & 1.000 (0.000) & 0.000 (0.000) & 0.023 (0.002) \\
DepMap CRISPR & 100 & 0.271 (0.001) & 1.000 (0.000) & 0.000 (0.000) & 0.078 (0.003) \\
DepMap CRISPR & 200 & 0.225 ($9.76\times10^{-4}$) & 1.000 (0.000) & 0.000 (0.000) & 0.172 (0.002) \\
DepMap CRISPR & 500 & 0.187 ($5.68\times10^{-4}$) & 1.000 (0.000) & $4.78\times10^{-6}$ ($4.78\times10^{-6}$) & 0.244 (0.001) \\
DepMap CRISPR & 1000 & 0.168 ($3.87\times10^{-4}$) & 1.000 (0.000) & $5.68\times10^{-6}$ ($5.68\times10^{-6}$) & 0.276 ($7.44\times10^{-4}$) \\
T-cell Perturb-seq & 50 & 0.333 (0.004) & 1.000 (0.000) & 0.000 (0.000) & 0.004 (0.001) \\
T-cell Perturb-seq & 100 & 0.272 (0.003) & 1.000 (0.000) & 0.000 (0.000) & 0.018 (0.002) \\
T-cell Perturb-seq & 200 & 0.225 (0.002) & 1.000 (0.000) & 0.000 (0.000) & 0.050 (0.003) \\
T-cell Perturb-seq & 500 & 0.184 (0.001) & 1.000 (0.000) & 0.000 (0.000) & 0.101 (0.002) \\
T-cell Perturb-seq & 1000 & 0.168 ($8.37\times10^{-4}$) & 1.000 (0.000) & 0.000 (0.000) & 0.120 (0.001) \\
\bottomrule
\end{tabular}%
}
\normalsize
\caption{Audit-size results. Each cell is mean (MC-SE); DepMap uses
1000 seeds and T-cell Perturb-seq uses 200. Here
$\varepsilon_{\rm test}=\varepsilon_{\rm audit}=0.10$,
$\delta=0.01$, $\rho=0$, and the certificate level is 0.09. The audit is given the derived
$\kappa=0.977$ for the declared class $\Delta=2$; the simulated channel has sensitivity $0.998$
and clean false-alarm rate $\alpha_R=0.023$.}
\label{tab:supp-audit}
\end{table}

For DepMap, the oracle cap has power 0.390 and the fixed blind
$\bar\varepsilon=0.30$ cap has power 0.126. With 1000 controls, Audit-CADS
has mean final bound 0.168, simultaneous coverage 1.000, FDR
$5.68\times10^{-6}$, and power 0.276. For T-cell Perturb-seq, the corresponding
coverage and power are 1.000 and 0.120. Small audit sets remain conservative,
especially because the first test round has cap one: at $|A|=50$ DepMap power is
0.023, below the fixed blind cap.

Substituting $\kappa=1$ and a false-alarm-free detector raises $|A|=1000$ power
to 0.313 and lowers the mean final bound to 0.141. We report that configuration
only to price the assumption: no threshold detector attains it at any finite
$\tau$, since corruption arriving just below $\tau$ is never flagged. The
$0.037$ power difference is what the earlier, idealized channel was silently
borrowing.

\begin{table}[t]
\centering
\small
\begin{tabular}{rrrrrrr}
\toprule
$\kappa$ & $\rho$ & $\varepsilon_{\rm audit}$ & Final bound (SE) & Coverage & FDR & Power (SE) \\
\midrule
0.9 & 0.02 & 0.08 & 0.197 (0.002) & 1.000 & 0.000 & 0.214 (0.004) \\
0.9 & 0.05 & 0.05 & 0.186 (0.002) & 1.000 & 0.000 & 0.233 (0.003) \\
0.9 & 0.10 & 0.00 & 0.143 (0.000) & 1.000 & 0.000 & 0.288 (0.001) \\
0.7 & 0.02 & 0.08 & 0.214 (0.002) & 1.000 & 0.000 & 0.175 (0.005) \\
0.7 & 0.05 & 0.05 & 0.201 (0.002) & 1.000 & 0.000 & 0.203 (0.004) \\
0.7 & 0.10 & 0.00 & 0.155 (0.000) & 1.000 & 0.000 & 0.271 ($9.73\times10^{-4}$) \\
0.5 & 0.02 & 0.08 & 0.241 (0.003) & 1.000 & 0.000 & 0.106 (0.006) \\
0.5 & 0.05 & 0.05 & 0.224 (0.003) & 1.000 & 0.000 & 0.141 (0.006) \\
0.5 & 0.10 & 0.00 & 0.177 (0.000) & 1.000 & 0.000 & 0.226 ($9.66\times10^{-4}$) \\
\bottomrule
\end{tabular}
\normalsize
\caption{Boundary-assumption stress at $|A|=200$ on DepMap
(200 seeds). Each mean with uncertainty is shown as mean (MC-SE).
The generator sets
$\varepsilon_{\rm audit}=\max(0,\varepsilon_{\rm test}-\rho)$ and
$\kappa_{\rm actual}=\kappa$; therefore increasing $\rho$ also lowers the
simulated audit corruption. Rows must not be interpreted as empirical
validation of $\kappa$ or $\rho$.}
\label{tab:supp-sensitivity}
\end{table}

Table~\ref{tab:supp-sensitivity} deliberately evaluates the assumed boundary
$\varepsilon_{\rm test}=\varepsilon_{\rm audit}+\rho$. The nonmonotone-looking
power comparison across $\rho$ is induced by changing the audit data-generating
rate, not by a claim that larger mismatch improves performance.

\section{Clean-Null Misspecification Sensitivity}
\label{sec:nullstress}

Everything CADS declares rests on one inequality. The cap solves
$(1-\varepsilon)\,\mathbb E_{f_0}[\min(L,C^*)]+\varepsilon C^*=1$ for an assumed
boundary null $f_0=N(0,s^2)$, so with the cap built from $\bar\varepsilon$ and
contamination arriving at $\varepsilon$, a null read satisfies
\[
\mathbb E_{\rm null}[e]\le(1-\varepsilon)\,\mathbb E_{f_0'}[\min(L,C^*(\bar\varepsilon))]
+\varepsilon\,C^*(\bar\varepsilon),
\]
the worst case being contamination that saturates the cap. Validity needs this at most
one. Table~\ref{tab:supp-nullstress} measures it by Monte Carlo ($4\times10^6$ draws
per cell) for a family of misspecified true nulls $f_0'$. A cell is marked broken only
when its $95\%$ lower endpoint clears one.

\begin{table}[h]
\centering
\small
\setlength{\tabcolsep}{2pt}
\begin{tabular}{lrrrr}
\toprule
true clean null $f_0'$ & $\bar\varepsilon{=}0.10$ & $0.15$ & $0.20$ & $0.30$ \\
\midrule
assumed $N(0,s^2)$        & 1.0002 & 0.9043 & 0.8405 & 0.7557 \\
\midrule
null mean $+0.05$         & 1.0723$\dagger$ & 0.9687 & 0.8994 & 0.8069 \\
null mean $+0.10$         & 1.1491$\dagger$ & 1.0366$\dagger$ & 0.9610 & 0.8600 \\
null mean $+0.20$         & 1.3154$\dagger$ & 1.1816$\dagger$ & 1.0913$\dagger$ & 0.9703 \\
null mean $+0.30$         & 1.4963$\dagger$ & 1.3363$\dagger$ & 1.2281$\dagger$ & 1.0838$\dagger$ \\
\midrule
null sd $\times1.1$       & 1.0463$\dagger$ & 0.9400 & 0.8693 & 0.7760 \\
null sd $\times1.5$       & 1.1957$\dagger$ & 1.0510$\dagger$ & 0.9568 & 0.8353 \\
null sd $\times2.0$       & 1.3202$\dagger$ & 1.1400$\dagger$ & 1.0250$\dagger$ & 0.8796 \\
\midrule
standardized $t_{10}$     & 0.9760 & 0.8830 & 0.8217 & 0.7408 \\
standardized $t_{5}$      & 0.9440 & 0.8550 & 0.7969 & 0.7211 \\
standardized $t_{3}$      & 0.8782 & 0.7976 & 0.7460 & 0.6799 \\
heteroscedastic $0.5$--$1.5\times$ & 0.9726 & 0.8786 & 0.8170 & 0.7363 \\
\bottomrule
\end{tabular}
\normalsize
\caption{Clean-null misspecification. Implied $\mathbb E_{\rm null}[e]$ at true
$\varepsilon=0.10$; validity needs $\le 1$ and $\dagger$ marks broken. Location and
scale families perturb the reference-null pool only.}
\label{tab:supp-nullstress}
\end{table}

Three implications follow. First, at its own operating point
($\bar\varepsilon=\varepsilon$) the cap is \emph{exactly} tight: the assumed row reads
$1.0002$, because the budget is solved as an equality. There is no margin there at all,
so any misspecification that raises the clean term breaks validity immediately: a
null mean $0.05$ above the assumed boundary, or a null sd underestimated by ten
percent, is already enough.

Second, the shape families we tested are benign, but shape in general is not, and the
distinction matters more than the measurement. Heavier tails and heteroscedastic scale
both \emph{lower} $\mathbb E_{\rm null}[e]$ here ($t_3$ reaches $0.8782$): the cap
truncates the upper tail while the lower tail drives $L$ toward zero, and for these
families the second effect wins. That is not a theorem about heavy tails. Fixing the
mean at $0$ and the variance at $s^2$ still leaves enough freedom to break the
inequality: put mass $0.59924$ at $0$, $0.2$ at each of $\pm0.76$, and $0.00038$ at
each of $\pm5$, giving kurtosis $9.7$, and $\mathbb E_{f_0'}[\min(L,C^*)]=0.8654$
against the Gaussian's $0.7541$, so the implied $\mathbb E_{\rm null}[e]=1.100>1$. The
capped e-value is not shape-robust; it is robust to the two shape families we happened
to test, and the Gaussian clean null is load-bearing. What carries into deployment is
the sensitivity to mean and scale that the table measures, and only inside a declared
clean-law family.

Third, over-stating $\varepsilon$ is the only headroom a frozen cap has, and the table
prices it. Taking $\bar\varepsilon=0.20$ absorbs a null mean shift of $0.10$ or a
$1.5\times$ scale error; $\bar\varepsilon=0.30$ absorbs $0.20$ or $2.0\times$. This is the
operational content of over-rating being safe: it is not merely harmless, it is the only
defence a fixed cap has against a mis-located or mis-scaled null. It is not a defence
against the wrong clean-law family, which instead requires declaring the wider family and
re-solving the cap for its worst case.

\paragraph{What a broken budget costs.}
Losing the precondition is not the same as losing control. Rerunning the DepMap screen
(200 paired seeds, $\bar\varepsilon=\varepsilon=0.10$, $q=0.10$) with the null pool
shifted, CADS's empirical FDR moves from $0.0001$ (assumed) to $0.0016$ at $+0.15$ and
$0.0126$ at $+0.30$, while naive BH moves from $0.372$ to $0.523$ and $0.653$. Power is
unchanged throughout ($0.3956$--$0.3960$), which confirms the perturbation is acting on
the null pool alone. So the degradation is real and monotone (the violation is not
cosmetic), yet even at twice the breaking shift the realized FDR is still eight times
under $q$. We claim the first fact and not the second: past the boundary the guarantee
is void, and the remaining margin is empirical, not certified.

\section{Unmodeled Contamination Channels}
\label{sec:contamstress}

Theorem~A.1 quantifies over an arbitrary contaminating density, allowed to depend on the
history, subject only to a per-read conditional rate bound. The frozen protocol
exercises one member of that class, and an unusually kind one: i.i.d., so the
conditional rate equals the marginal; one-directional; almost always cap-saturating;
and far above the audit's threshold, so the audit sees it. A reviewer is entitled to
ask what happens for the unkind members, so we ran them. Every channel below holds the
\emph{marginal} rate at $\varepsilon=0.10$ and differs only in what else it does.
Table~\ref{tab:supp-contamstress} reports the resulting FDR and power.

\begin{table}[t]
\centering
\small
\begin{tabular}{llrrr}
\toprule
Channel & Bound & CADS FDR & Power & Naive FDR \\
\midrule
frozen $\mathrm{U}(2,4)$ (reference) & holds & 0.0001 & 0.3956 & 0.3720 \\
cap-saturating point mass            & holds & 0.0001 & 0.3958 & 0.0083 \\
subthreshold (just under $\tau$)     & holds & 0.0002 & 0.3962 & 0.0204 \\
mixed-sign                           & holds & 0.0000 & 0.2435 & 0.1320 \\
history-dependent (re-reads)         & holds & 0.0095 & 0.4736 & 0.7465 \\
target-specific (near boundary)      & holds & 0.0009 & 0.4327 & 0.4304 \\
clustered (one coin per campaign)    & \textbf{violated} & 0.0474 & 0.3991 & 0.0500 \\
\bottomrule
\end{tabular}
\normalsize
\caption{Contamination stress, DepMap, 200 paired seeds per channel,
$q=0.10$. All channels have marginal rate $0.10$. Only the last violates the per-read
\emph{conditional} bound the theorem requires.}
\label{tab:supp-contamstress}
\end{table}

The six channels that respect the conditional bound all hold, at worst $0.0095$ against
$q=0.10$, and they are not gentle: corruption placed exactly at the cap threshold,
corruption hidden just below the audit's detector, corruption aimed only at nulls near
the boundary, and corruption that fires only on candidates the agent chose to re-read,
which blinded controls, never re-read, cannot see at all. This is the theorem behaving as
stated: the cap bounds each increment by $C^*$ whatever the contaminating law is, so
shape, sign, position and history-dependence do not reach validity. They reach power
instead. Mixed-sign corruption costs $0.15$ of power because downward corruption destroys
signal the cap was never protecting.

The last row is the one that matters. Flipping a single coin per campaign and then
corrupting everything or nothing keeps the marginal rate at $0.10$ and makes the
conditional rate zero or one, which is exactly what \eqref{eq:supp-mismatch} forbids. Its
FDR rises by two orders of magnitude, to $0.0474$. At this
operating point it still does not exceed $q$ (the $95\%$ interval is
$[0.0201,0.0747]$), so this is not a demonstration that the certificate fails, it is a
demonstration that the assumption is load-bearing and that the empirical margin absorbs
its violation here. The same asymmetry appears under clean-null misspecification
(\S\ref{sec:nullstress}): losing the precondition is not the same as losing control, and
we claim only the precondition.

Two things this does not establish. The realized FDR at the primary point is roughly
$750\times$ below $q$, which is consistent with a loose certificate in this
signal-to-noise regime rather than a tight one; and the only valid comparator declares
nothing at the matched budget, so no similarly powered valid method demonstrates that
this conservatism buys a useful robustness--power frontier.

\section{Proposer Contribution Analyses}
\label{sec:llm}

This section supports ``What the Certificate Measures'' in the main paper. It reads three open-weight
discovery loops through the certificate to measure what each proposer adds, and reports every figure the
main text summarizes. Two of the three were fixed in the original frozen contract below; the third,
Meta-Llama-3.1-8B-Instruct, was added afterwards on the identical protocol and is reported as an
extension rather than as part of that freeze.

\paragraph{Protocol.}
The measurement uses the full DepMap 26Q1 ACH-000446 registry
($m=18{,}159$; $1{,}822$ reference hits), \emph{not} the $6{,}000$-gene subsample the
validity sweeps of \S\ref{sec:results} use, so the two are never compared across the
boundary. The contract was frozen before the run: two models
(Qwen2.5-7B-Instruct, Mistral-Nemo-Instruct-2407), three feedback modes, two prompt
modes, paired seeds $0$--$19$, $12$ rounds, batch $40$ ($480$ reads over $260$ observed
candidates), re-observation fraction $0.5$, $\varepsilon=q=0.10$, temperature $0.7$,
top-$p$ $0.9$, at most $700$ new tokens and two generation attempts per round.
Under-delivery is filled by a seeded, label-free random draw over eligible unseen
candidates; oracle completion is forbidden and the per-round raw, valid, invalid,
duplicate, and fallback counts are retained. The \emph{thin} prompt ships with the loop;
the \emph{informed} prompt additionally lists the already-tested candidates and asks for
unseen ones, changing nothing else. The pre-specified claim gate is that the paired
two-sided $95\%$ $t$ interval on \texttt{cads.true} has a lower endpoint above zero for
both full-minus-none and full-minus-permuted.

\paragraph{Harness validation by replay.}
Before trusting the no-LLM control and the attribution decomposition, we verify that
the analysis harness reproduces the GPU run exactly. A CPU replay re-runs each frozen
arm from the stored per-round \texttt{responses}, keyed on
$(\text{seed},\text{round},\text{attempt})$ against the recorded \texttt{generation\_seed},
and re-derives every certified declaration. It reproduces all $40$ thin campaigns with
zero mismatches for both models. The no-LLM control and the origin decomposition
therefore run the identical certified-declaration code the GPU run used; the only change
is which mechanism fills a slot.

\paragraph{The no-LLM control.}
An identical loop whose proposer returns nothing (the random fallback fills every
fresh slot) certifies $2.60$ true discoveries under full feedback and $0.00$ under
both absent and permuted feedback. The paired $95\%$ $t$ interval is $[1.35,3.85]$
against each, so the preregistered gate the field runs passes with no LLM in the loop.
The non-LLM re-selection policy re-observes candidates that read high, deepening their
evidence, which is what clears the gate.
Table~\ref{tab:supp-llm} reports the complete proposer-by-feedback grid.

\begin{table}[t]
\centering
\small
\begin{tabular}{lllrrrr}
\toprule
Prompt & Model & Feedback & LLM slots & CADS true & CADS FDR & Naive FDR \\
\midrule
thin & Qwen2.5-7B   & full     & 57.7  & 2.95 & 0.000 & 0.228 \\
thin & Qwen2.5-7B   & none     & 257.5 & 0.30 & 0.000 & 0.522 \\
thin & Qwen2.5-7B   & permuted & 55.4  & 0.00 & 0.000 & 0.622 \\
thin & Mistral-Nemo & full     & 131.5 & 2.95 & 0.000 & 0.154 \\
thin & Mistral-Nemo & none     & 196.8 & 0.25 & 0.000 & 0.509 \\
thin & Mistral-Nemo & permuted & 167.9 & 0.00 & 0.000 & 0.407 \\
\midrule
informed & Qwen2.5-7B   & full     & 175.8 & 5.50 & 0.000 & 0.155 \\
informed & Qwen2.5-7B   & none     & 237.8 & 0.45 & 0.000 & 0.628 \\
informed & Qwen2.5-7B   & permuted & 189.2 & 0.00 & 0.000 & 0.523 \\
informed & Mistral-Nemo & full     & 190.2 & 3.90 & 0.000 & 0.135 \\
informed & Mistral-Nemo & none     & 186.7 & 0.10 & 0.000 & 0.447 \\
informed & Mistral-Nemo & permuted & 191.0 & 0.00 & 0.000 & 0.371 \\
\midrule
thin & Llama-3.1-8B & full     & 161.8 & 5.40 & 0.000 & 0.129 \\
thin & Llama-3.1-8B & none     & 242.6 & 0.65 & 0.000 & 0.336 \\
thin & Llama-3.1-8B & permuted & 161.3 & 0.00 & 0.000 & 0.358 \\
informed & Llama-3.1-8B & full     & 251.7 & 5.65 & 0.000 & 0.150 \\
informed & Llama-3.1-8B & none     & 253.4 & 0.40 & 0.000 & 0.469 \\
informed & Llama-3.1-8B & permuted & 256.1 & 0.00 & 0.000 & 0.408 \\
\bottomrule
\end{tabular}
\normalsize
\caption{Full proposer grid ($m=18{,}159$; 20 paired seeds per cell).
``LLM slots'' is the mean number of 260 fresh slots filled with valid unseen
symbols; a seeded label-free draw fills the remainder. CADS FDR is mean FDP
and equals $0.000$ in every cell; naive BH ranges from $0.129$ to $0.628$.
Qwen and Mistral comprise the frozen contract; Llama-3.1-8B is an extension
under the same protocol. The shipped verifier binds every cell to its frozen
shard summary.}
\label{tab:supp-llm}
\end{table}

\paragraph{Displacement.}
Attributing each thin-prompt certified discovery to the mechanism that first proposed
it (re-observation deepens evidence but never originates) gives, per campaign, Qwen
$2.15$ LLM-originated $+\,0.80$ fallback $=2.95$, and Mistral-Nemo $2.35+0.60=2.95$; the
replay above confirms this decomposition. Read alone, the majority is LLM-originated. But
the fallback that fills the other $202$ (Qwen) or $129$ (Nemo) slots finds $0.80$ and
$0.60$, while the no-LLM loop's fallback fills all $260$ and finds $2.60$. The LLM's
originated discoveries therefore displace $1.80$ (Qwen) and $2.00$ (Nemo) the fallback
would otherwise have made, and each model's net over the identical no-LLM loop is only
$+0.35$ $[-1.05,1.75]$ and $+0.35$ $[-1.53,2.23]$, both indistinguishable from zero.

\paragraph{The prompt mechanism.}
The thin prompt is never told which candidates were already tested, so it re-proposes
genes it was shown read high and is scored a duplicate. The informed prompt fixes only
this. Valid unique symbols per campaign then rise from $57.7$ to $175.8$ for Qwen and
$131.5$ to $190.2$ for Mistral-Nemo, and Qwen's fallback share falls from $202$ to $84$
slots. The net over the no-LLM loop moves with delivery: Qwen from $+0.35$ to $+2.90$
$[1.26,4.54]$, a prompt effect of $+2.55$ $[1.40,3.70]$, both excluding zero;
Mistral-Nemo from $+0.35$ to $+1.30$ $[-0.62,3.22]$, prompt effect $+0.95$ $[-1.27,3.17]$,
neither excluding zero. In this initial 20-seed comparison, the interval excludes zero
for one model and remains inconclusive for the other.

\paragraph{Gate passage.}
Every LLM arm also clears the preregistered gate: under the thin prompt, Qwen
full-minus-none is $2.65$ $[1.48,3.82]$ and full-minus-permuted $2.95$ $[1.88,4.02]$;
Mistral-Nemo is $2.70$ $[1.53,3.87]$ and $2.95$ $[1.79,4.11]$; under the informed prompt
the same comparisons rise to $5.05$--$5.50$ (Qwen) and $3.80$--$3.90$ (Nemo). So does the
no-LLM control. Passing the gate is thus consistent with a real proposer, an absent
proposer, and a duplicate-ridden one alike; it does not by itself measure the LLM.

\paragraph{Budget.}
The low absolute yield is the read budget. Running the same no-LLM loop at $1\times$,
$2\times$, and $4\times$ the reads (more rounds, same batch) certifies $2.60$, $11.30$,
and $20.20$ true discoveries, with the gate holding throughout (full-minus-none
$[1.35,3.85]$, $[9.89,11.61]$, $[17.89,19.91]$). Yield grows over these three budgets on the same $20$ seeds; we did not
repeat the sweep on fresh seeds, so this does not by itself rule out a seed-sample effect.

\paragraph{Limits.}
Absolute power is $0.0016$--$0.0031$: the models certify roughly three to six of the
$1{,}822$ reference hits, because even the informed prompt emits a limited number of
valid unseen symbols per round. The LLM slot share varies across conditions by more than
a factor of four, so the passing contrasts are between campaigns, not between proposal
qualities holding delivery fixed. Permuted feedback yields exactly zero certified
discoveries in every seed for all three models, so each full-minus-permuted interval is driven
by its full arm. The measurement covers three open-weight models on one semi-synthetic,
DepMap-backed loop at $20$ seeds; it is model-dependent on a small sample.

\paragraph{Determinism.}
The run is deterministic given its seeds: the replay harness above reproduces all $40$
thin campaigns from their stored responses with zero mismatches, and
\path{verify_paper_measurement_numbers.py} runs that replay as one of its own checks
rather than reading the attribution it produces.

\paragraph{A strong non-LLM prior.}
The random fallback is a weak non-LLM baseline. A competent one ranks candidates by a truth-isolated
pan-essentiality prior: the mean gene effect $u=-\mathrm{effect}-0.5$ (higher is more essential, the
same continuous signal the target uses to define a hit) across the five DepMap 24Q4 lines \emph{other}
than the target ACH-000446, so even a different release of the target line cannot leak its own
dependency; genes absent from the panel take the pool minimum. The proposer fills each round's fresh
slots with the top unseen candidates by this prior and is otherwise identical to the LLM runs (same
registry, reads, exploit re-selection, noise tape, and seeds); being a predictable acquisition it is
covered by the terminal FDR guarantee. It is deterministic and CPU-only, and the continuous effect
avoids the ties a saturating dependency probability would create, so the ranking does not depend on
tie-breaking. The prior certifies $12.65$ true discoveries at empirical FDR $0$, a paired net $+10.05$
$[8.69,11.41]$ ($t$; bootstrap $[8.80,11.30]$) over the no-LLM loop, far above all three models.
\path{revision/llm_experiment/external_prior_baseline.py} writes
\path{revision/results/external_prior_v2.json}, which \path{verify_paper_measurement_numbers.py} binds
to the paper.

\paragraph{A third model.}
Meta-Llama-3.1-8B-Instruct (4-bit GGUF) runs the same protocol as Qwen and Mistral-Nemo: the v2 registry,
$12$ rounds of batch $40$, $\varepsilon{=}q{=}0.10$, paired seeds $0$--$19$, and thin and informed prompts
under full/none/permuted feedback. It certifies $5.40$ (thin) and $5.65$ (informed) at empirical FDR $0$,
netting $+2.80$ $[0.50,5.10]$ and $+3.05$ $[1.12,4.98]$ over the no-LLM loop (both excluding zero), and its
feedback gate holds (full minus none and full minus permuted intervals span $[2.75,7.29]$ across arms).
Unlike Qwen, its thin prompt already nets positive and its prompt effect is $+0.25$ $[-1.80,2.30]$,
an interval that includes zero, so the three models show three distinct patterns and the prompt effect is
model-dependent. \path{revision/llm_experiment/llama_measurement.py} reads the shipped per-(prompt,feedback)
shards and writes \path{revision/results/llama_v2.json}, which \path{verify_paper_measurement_numbers.py}
    binds to these values.

\paragraph{Quality versus delivery.}
The informed prompt raises the number of valid candidates delivered and could also raise their quality.
To separate the two we pre-specified a follow-up in a design file
(\path{revision/llm_experiment/FIXED_DELIVERY_DESIGN.md}, frozen before the fresh held-out seeds were run
and shipped with the code, though not lodged with a third-party registry) holding every proposer to an
identical per-round quota of $K$ valid, unseen, unique candidates. The intent is that only \emph{which}
candidates are chosen may differ; the delivery accounting below shows that the thin prompt cannot fill
the quota, so the arms end up matched on the quota requested rather than on the delivery achieved, and
we read the result accordingly. Three changes tighten the contrast relative to the
frozen runs: $50$ fresh generation seeds ($20$--$69$, disjoint from the $0$--$19$ used to set the quota);
a candidate-level noise tape in which the $j$-th read of candidate $i$ is drawn from a stream keyed by
$(\text{seed},i,j)$, so every arm sees the identical readout for the same (candidate, visit); and a
common per-round fallback priority shared by all arms. Truncation to the quota is by generation order
only, and under-delivery is recorded as a delivery failure and filled from the common fallback, never
from oracle truth.

\paragraph{Setting the quota.}
A calibration on seeds $0$--$19$ with the quota released measured achievable delivery. Round $0$
saturates at $40$ for every model and prompt, but in rounds $1$--$11$ Qwen-thin achieves a median of
$0$ new valid unique candidates even at $8$ generation attempts (mean $4.25$), against $17.3$ for
Qwen-informed, $12.9$ and $16.6$ for Mistral thin/informed, and $17.8$ and $19.9$ for Llama. Qwen-thin
is therefore the binding arm, and we froze a single global $K{=}5$ with $N_{\max}{=}8$ before touching
the held-out seeds; round $0$ is left uncapped because it is already matched.

\paragraph{Delivery accounting.}
No model complies with the quota under the thin prompt, and how badly it fails is model-dependent:
thin misses $40.7$ of the $55$ late-round quota slots for Qwen, $26.0$ for Mistral-Nemo and $13.1$ for
Llama, while the informed prompt closes the gap in all three (at most $0.7$ slots missed). The quota
therefore equalises what each proposer is \emph{asked} for, not what it delivers, and the one effect
present in every model is the delivery difference itself. The certified count does not move in a common
direction: informed is ahead for Qwen and Mistral-Nemo and behind for Llama. Neither does per-slot
reference-hit precision, which is lower under the informed prompt for Qwen and Llama and higher for
Mistral-Nemo. Precision is also measured over different numbers of slots in the two arms, because thin
fills the slots it manages to fill first, so we read it as descriptive rather than as a like-for-like
quality contrast. In the frozen $20$-seed runs the informed prompt fills $175.8$ of $260$ fresh slots
against thin's $57.7$ at precision $0.226$ versus $0.215$ (paired difference $+0.010$ $[-0.037,0.057]$,
one model).

\begin{center}
\small
\setlength{\tabcolsep}{4pt}
\begin{tabular}{llcccc}
\toprule
proposer & prompt & certified & slots & failures & prec. \\
\midrule
Qwen & thin & $3.58$ & $54.3$ & $40.7$ & $0.246$ \\
Qwen & informed & $4.32$ & $94.3$ & $0.7$ & $0.238$ \\
Mistral-Nemo & thin & $3.42$ & $69.0$ & $26.0$ & $0.279$ \\
Mistral-Nemo & informed & $3.58$ & $94.4$ & $0.6$ & $0.304$ \\
Llama & thin & $5.08$ & $81.9$ & $13.1$ & $0.422$ \\
Llama & informed & $4.44$ & $95.0$ & $0.0$ & $0.252$ \\
\midrule
random & --- & $1.74$ & $95.0$ & $0$ & $0.107$ \\
external prior & --- & $16.12$ & $95.0$ & $0$ & $0.937$ \\
\bottomrule
\end{tabular}
\end{center}
\noindent Each row is $50$ campaigns; the two non-LLM floors are model-independent, so they carry one
row each. No campaign among the $700$ produced a false declaration, so empirical FDR is $0.0000$
everywhere. A sample of zeros has zero standard error, so a normal-approximation bound would report an
upper limit of zero, which is not a bound. Because $\mathrm{FDP}\le1$ implies
$\mathbb E[\mathrm{FDP}]\le\Pr(\mathrm{FDP}>0)$, the Clopper--Pearson limit for the latter bounds the
former, giving a one-sided $95\%$ bound of $0.058$. The independent unit there is the generation seed
and not the campaign, because the three models share the same $50$ seeds and the same noise tape, so
pooling them does not buy $n{=}150$.

\paragraph{The prompt effect at a matched quota.}
Paired seed-cluster percentile bootstrap ($10{,}000$ resamples) is the primary interval; Holm corrects
across the three model-level contrasts.

\begin{center}
\small
\setlength{\tabcolsep}{4pt}
\begin{tabular}{lccc}
\toprule
model & natural & fixed & Holm \\
\midrule
Qwen & $+1.10\,[-0.10,2.28]$ & $+0.74\,[-0.12,1.66]$ & $0.337$ \\
Mistral & $+1.14\,[-0.14,2.44]$ & $+0.16\,[-0.98,1.34]$ & $0.790$ \\
Llama & $+0.34\,[-1.20,1.84]$ & $-0.64\,[-1.88,0.56]$ & $0.645$ \\
\bottomrule
\end{tabular}
\end{center}
\noindent No interval excludes zero. A pre-specified TOST equivalence test at a margin of one
declaration does not establish equivalence either ($p_{\text{TOST}}$ between $0.08$ and $0.29$), so the
result supports no detectable effect at this sample size, not a zero effect.

\paragraph{Seeds versus noise model.}
The natural prompt effect also attenuates on the fresh seeds, from the $+2.55$ $[1.40,3.70]$ reported at
$20$ seeds to $+1.10$ $[-0.10,2.28]$. Re-running the natural arms on the original seeds $0$--$19$ under
the same candidate-level tape returns $+2.00$ $[0.50,3.40]$, which still excludes zero. Changing the
noise model while holding the seeds fixed therefore costs about $0.55$, whereas changing the seed sample
while holding the noise model fixed leaves an interval that includes zero. We did not test the
difference between those two estimates, and one interval excluding zero while another does not is not
itself evidence that they differ, so we do not attribute the attenuation to either factor. What we
report is that the $20$-seed magnitude does not reproduce on fresh seeds.

\paragraph{Baselines at the same quota.}
Every LLM arm is ahead of the label-free floor at the same quota (informed minus random-fixed is
$+2.58$, $+1.84$ and $+2.70$ for Qwen, Mistral and Llama). The truth-isolated pan-essentiality prior
nevertheless dominates all of them, certifying $16.12$ against the strongest LLM row's $5.08$, with the
paired difference against each informed arm between $-11.68$ and $-12.54$. These six baseline contrasts
are exploratory and their intervals are unadjusted pointwise ones, unlike the three prompt-effect
contrasts above, which are the Holm family. Drivers
\path{revision/llm_experiment/fixed_delivery.py} and \path{fixed_delivery_measure.py} write
\path{revision/results/fixed_delivery_v2.json} and \path{fixed_delivery_diagnostic_v2.json}, which
\path{verify_paper_measurement_numbers.py} binds to the values quoted here.

\section{Reproducibility and Artifact Integrity}
\label{sec:repro}

\paragraph{Computing infrastructure.}
The statistical experiments (the primary grid, cross settings, strict baseline, Audit-CADS, and both
stress grids) run on CPU under Python~3.13, NumPy~2.4, and SciPy~1.17, with no GPU and no build step.
The optional open-weights proposer run uses a single 46\,GB data-center GPU with PyTorch, Transformers,
and Accelerate, loading the public checkpoints \path{Qwen/Qwen2.5-7B-Instruct},
\path{mistralai/Mistral-Nemo-Instruct-2407}, and \path{meta-llama/Llama-3.1-8B-Instruct} (the last as a
4-bit Q4\_K\_M GGUF conversion). Two limits on that step are worth stating: \path{requirements-llm.txt}
gives minimum versions rather than an exact lock, and the shards identify each checkpoint by name and
not by commit revision, so re-running generation reproduces the protocol but need not reproduce the
token stream. No GPU is needed to check any reported number: the gates below recompute each from the frozen
seed-level shards, except the quota calibration of Appendix~\ref{sec:llm}, whose own runs are not
shipped and which is therefore checked against the frozen design file instead.

\paragraph{One-command gates.}
The command \texttt{python3}
\path{revision/validate_corrected_results.py} rejects missing
campaigns, duplicate seeds, incomplete $(q,\varepsilon)$ cells, mismatched
full-family sizes, aggregate/seed-row disagreement, stale source hashes, and
incomplete Audit records. Figure and table generation begins by invoking this
gate. The command \texttt{python3}
\path{revision/render_corrected_figures.py} creates all
data-derived panels and tables; then
\texttt{python3} \path{revision/sync_paper_artifacts.py} copies only validated
artifacts into the paper tree and records their SHA-256 hashes.

\paragraph{Result files.}
\path{revision/results/corrected_primary_grid.json} contains 5000
seed-level primary campaigns.
\path{revision/results/corrected_cross_00.json} through
\path{corrected_cross_07.json} contain 1600 cross-setting campaigns.
\path{corrected_zhou_liu_algorithm2.json} contains 200 strict-baseline
campaigns. \path{phase5/cads_certified_eps_summary.json} contains the
primary Audit-CADS runs, cross-reference audit runs, sensitivity runs,
per-seed records, parameters, and code/data hashes. The proposer measurement
(Appendix~\ref{sec:llm}) is under
\path{experiments/20260717-llm-pool-v2} (thin) and
\path{experiments/20260717-llm-pool-v2-informed}, each a $120$-campaign frozen
grid with per-round delivery counts, response traces, and
\texttt{generation\_seed}s; \path{revision/results/FINAL_NUMBERS_v2.json}
collects the reported no-LLM control, attribution, prompt-effect, and budget
aggregates, and the CPU replay/control drivers live under
\path{revision/llm_experiment/}.

\paragraph{Tests.}
\texttt{python3 -m unittest}\\
\texttt{tests.test\_statistical\_contract tests.test\_fixed\_delivery} executes all $62$: $48$ contract tests
and $14$ matched-quota tests, the latter covering the noise tape's frozen
$\mathrm{Uniform}(2,4)$ corruption law, quota enforcement, truncation by
generation order, no-oracle-rescue, and the per-model aggregation contract.
Both modules must be named, because \path{tests/} carries no
\path{__init__.py} and test discovery therefore finds neither. The result
validator recomputes all reportable means and MC-SE values from seed rows
rather than trusting stored aggregates, over the primary grid, the eight cross
settings, the strict baseline, Audit-CADS, and both stress grids; the
paper-number verifier additionally rebuilds the matched-quota summaries from
their $700$ raw per-seed shards and fails if those shards are absent.

\paragraph{Scope of the released evidence.}
The reportable corrected evidence consists of 6800 primary, cross-setting, and
strict-baseline campaigns, plus 4800 Audit-CADS primary and sensitivity seed
records. The proposer measurement adds $240$ frozen LLM campaigns in the
original two-model contract (two models $\times$ two prompt modes $\times$
three feedback modes $\times$ $20$ paired seeds), $120$ more for the Llama
extension on the same grid, $60$ no-LLM control and $20$ external-prior
campaigns at the same seeds, and the matched-quota follow-up's $700$
follow-up and $40$ diagnostic campaigns: $1{,}180$ in all. The independent
CPU replay from stored responses covers the $40$ thin campaigns of the original
contract, which is what validates the analysis harness; every other reported
number, apart from the quota calibration whose own runs are not shipped, is
recomputed from the stored seed-level shards by
\path{verify_paper_measurement_numbers.py}, which ships with them.

\bibliographystyle{plainnat}
\bibliography{refs}
\end{document}